\pdfoutput=1

\documentclass[twoside]{article}

\usepackage[preprint]{aistats2027}

\usepackage{amsmath,amssymb,amsthm,mathtools}
\usepackage{graphicx}
\usepackage{booktabs}
\usepackage{multirow}
\usepackage{xcolor}
\usepackage{caption}
\usepackage{subcaption}
\usepackage{algorithm}
\usepackage{algpseudocode}
\usepackage{balance}
\usepackage{placeins}
\usepackage{needspace}
\usepackage{url}

\usepackage[round]{natbib}

\usepackage[colorlinks=true,linkcolor=black,citecolor=black,urlcolor=blue]{hyperref}
\usepackage[capitalise]{cleveref}

\theoremstyle{plain}
\newtheorem{theorem}{Theorem}
\newtheorem{proposition}[theorem]{Proposition}

\newtheorem{corollary}[theorem]{Corollary}
\theoremstyle{definition}

\newcommand{\R}{\mathbb{R}}
\newcommand{\T}{\mathbb{T}}
\newcommand{\E}{\mathbb{E}}

\newcommand{\norm}[1]{\left\lVert #1 \right\rVert}

\DeclareMathOperator*{\argmin}{arg\,min}
\DeclareMathOperator{\MMD}{MMD}

\newcommand{\spline}{\textsc{Spline-only}}

\providecommand{\Jcap}{\mathcal J}

\makeatletter\newsavebox{\nn@tabbox}
\newcommand{\nnfit}[1]{\sbox{\nn@tabbox}{#1}\ifdim\wd\nn@tabbox>\linewidth\resizebox{\linewidth}{!}{\usebox{\nn@tabbox}}\else\usebox{\nn@tabbox}\fi}\makeatother

\newcommand{\ourhead}[1]{\shortstack[c]{\textbf{#1}\\{\scriptsize(ours)}}}
\crefname{proposition}{Proposition}{Propositions}
\Crefname{proposition}{Proposition}{Propositions}

\crefname{corollary}{Corollary}{Corollaries}
\Crefname{corollary}{Corollary}{Corollaries}

\begin{document}

\runningtitle{Neural Wasserstein Splines: Representation-Invariant Second-Order Priors for Trajectory Inference}
\runningauthor{Amartya Banerjee, Harlin Lee, Caroline Moosmüller}

\twocolumn[

\aistatstitle{Neural Wasserstein Splines: Representation-Invariant Second-Order Priors for Trajectory Inference}

\aistatsauthor{%
Amartya Banerjee \And
Harlin Lee \And
Caroline Moosmüller}

\aistatsaddress{%
\footnotesize Dept. of Computer Science\\ UNC Chapel Hill \And
\footnotesize SDSS\\ UNC Chapel Hill \And
\footnotesize Dept. of Mathematics\\ UNC Chapel Hill
}

]

\begin{abstract}
    Population dynamics are often observed only as unpaired snapshots, which leave the evolution between measurements underdetermined. Second-order regularization can constrain it, but a penalty on the accelerations of sampled particles is a prior on joint particle paths, and different particle representations of the same marginal curve can have different acceleration costs. We therefore introduce Neural Wasserstein Splines (NWSP), a sample-based method that learns a continuous population curve under a Wasserstein smoothness prior. A neural network maps latent variables to the control points of a cubic spline in time, so samples at any time come from the spline basis without numerical integration. We project the generator's velocity onto gradients of a finite dictionary of scalar potentials and differentiate this measure-dependent projection in time, obtaining a finite-dimensional surrogate for the intrinsic Wasserstein acceleration, trained jointly with distributional fit losses. For a fixed dictionary and ridge parameter, we prove that the population penalty depends only on the marginal curve and bound its sensitivity to moment-estimation errors. Learned potentials (NWSP-N) sharpen acceleration estimates. On various synthetic and single-cell benchmarks, NWSP is highly competitive with transport and flow-based methods, ranking in the top two in 11 of 16 settings. The penalty improves on an unregularized spline in 13 of 16 settings and, while competitive with a particle-acceleration penalty, adds provable representation invariance.
\end{abstract}

\section{Introduction}
\label{sec:intro}

Observing a population at only a few times does not determine how it moves in between. A natural prior is that its motion changes gradually: penalizing acceleration favors nearly constant velocity, so the curve bends only where the data require it, whereas interpolating consecutive snapshots separately produces velocity jumps at every observation time. In single-cell experiments, destructive assays also mean that cells measured at successive times cannot be matched \citep{schiebinger2019optimal,tong2020trajectorynet}. Trajectory inference therefore seeks a continuous curve of distributions that explains the snapshots, and the dynamical prior determines the evolution between them.

\begin{figure*}[t]
\centering
\includegraphics[width=0.82\textwidth]{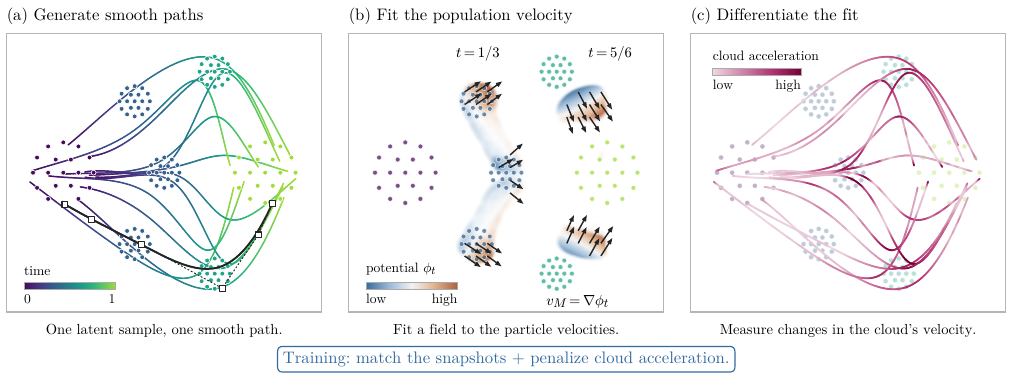}
\caption[Neural Wasserstein Splines]{\textbf{Learning a smooth population curve from unpaired snapshots.} (a) A network maps each latent sample to spline control points (squares), defining a complete path. (b) Ridge regression fits the velocity $v_M=\nabla\phi_t$ to particle velocities. Blue/orange indicates lower/higher potential $\phi_t$, scaled separately at each time; shading fades in low-density regions. (c) Differentiating the fitted field gives cloud acceleration $a_M$, whose magnitude colors the same paths. Training balances snapshot fit with mean squared cloud acceleration. NWSP-N also learns the potentials in (b).}
\label{fig:overview}
\end{figure*}

Applying an acceleration penalty to a population requires deciding what acceleration means. One option is the acceleration of the particles in a generative model, but unpaired snapshots do not identify which particle is which: many particle systems induce the same distribution at every time yet have very different accelerations (a Gaussian cloud can stay still or rotate rigidly; \cref{fig:rotation}). A particle-acceleration penalty then also penalizes an arbitrary modeling choice. The alternative is the intrinsic Wasserstein acceleration, the rate of change of the minimum-norm velocity that transports the curve of distributions, which depends only on the curve \citep{benamou2018second,justiniano2024approximation}. NWSP therefore regularizes the acceleration of the curve of distributions rather than that of any individual particle path.

Computing such curves is expensive: spline constructions need optimal transport (OT) couplings between consecutive snapshots \citep{chewi2021fast}, subdivision schemes such as Wasserstein Lane--Riesenfeld (WLR) solve repeated OT problems during refinement \citep{banerjee2025efficient}, and neural flow models generate samples by integrating a learned vector field. NWSP also uses splines, but over individual particle paths, so no transport problem is solved between snapshots; we compare scaling, not measured runtime (\cref{sec:cost}).

We introduce \emph{Neural Wasserstein Splines} (NWSP).\footnote{Code is available at \url{https://github.com/amartya21/Neural-Wasserstein-Splines.git}.} A neural network maps a random latent variable to a whole smooth particle path, a spline in time, so the population can be sampled at any time without integration or cross-time OT couplings. We estimate the population velocity from these paths, differentiate it in time to obtain the population acceleration, and train the network to fit the snapshots while keeping this acceleration small. NWSP-N additionally learns the potentials used to estimate the velocity.

\paragraph{Contributions.}
\textbf{(i) Generative splines without cross-time OT.} NWSP learns from minibatches of the full snapshots and samples at any time without numerical integration, using full training support on three datasets where exact-transport splines require subsampling.
\textbf{(ii) A differentiated population-acceleration surrogate.} We prove representation invariance for fixed potentials, bound the surrogate's sensitivity to moment errors, establish Fourier approximation rates, and show that omitting the moving-Gram correction leaves a nonvanishing bias.
\textbf{(iii) Adaptive potential estimation.} NWSP-N learns the potentials through a separate velocity-fitting objective; matched-size controls isolate the gain from adaptation and generator ablations evaluate the acceleration prior.

\subsection{Related work}
\label{sec:related}
\paragraph{Wasserstein splines.}
\citet{benamou2018second} formulate intrinsic second-order models and a relaxation over particle paths, and \citet{justiniano2024approximation} discretize intrinsic splines with local Wasserstein distances and generalized barycenters. \citet{chewi2021fast} interpolate in Euclidean space after constructing OT correspondences, while WLR and four-point schemes refine curves by repeated geodesic averaging \citep{banerjee2025efficient}. NWSP estimates population acceleration through a system indexed by potentials rather than observed particles.

\paragraph{Neural population dynamics.}
TrajectoryNet and MIOFlow fit flows with transport and geometric constraints \citep{tong2020trajectorynet,huguet2022manifold}. Flow and bridge matching regress fields against conditional-path targets \citep{lipman2023flow,tong2024improving,tong2024simulation}; Metric Flow Matching and PACE incorporate data geometry \citep{kapusniak2024metric,yu2026pace}. SplineFlow and MMFM use spline interpolants as flow-matching targets \citep{rathod2026splineflow,rohbeck2025modeling}; ALI-CFM learns interpolants adversarially before fitting a velocity field \citep{kviman2026multimarginal}. NWSP uses the spline itself to generate samples, with no second generative field.

\paragraph{Smoothness and velocity estimation.}
Momentum bridges and acceleration matching model position--velocity dynamics \citep{chen2023deep,theodoropoulos2025momentum,dazzini2026trajectory,sun2026tracingflow}; smooth bridges and Wassersplines impose smoothness through reference processes or field penalties \citep{hong2025trajectory,zhang2022wassersplines}. \citet{guan2026lagrangian} learn population mechanics from a Wasserstein action principle. NWSP instead uses population acceleration as a regularizer for a direct spline generator. Its velocity fit builds on Action Matching and Wasserstein Lagrangian Flows \citep{neklyudov2023action,neklyudov2024computational}; differentiating the measure-dependent projection yields its second-order penalty.

\section{Preliminaries}

\label{sec:problem}
At times $t_1<\cdots<t_T$, we observe point clouds $\{x_{k,i}\}_{i=1}^{n_k}\subset\R^d$ with empirical measures $\hat\rho_{t_k}=n_k^{-1}\sum_i\delta_{x_{k,i}}$, with no correspondence between samples observed at different times. We rescale time to $[0,1]$ and write $\mathcal R$ for the indices retained for training. We aim to infer a probability-measure curve $\rho_t$ with finite second moments by distributional, rather than pointwise or exact, fitting.

\paragraph{Intrinsic velocity and acceleration.}
For a sufficiently regular curve, the Wasserstein tangent velocity \(v_t(x)\) is the minimum-\(L^2(\rho_t)\)-norm solution of the continuity equation \citep{ambrosio2008gradient},
\begin{equation}
 \partial_t\rho_t+\nabla\!\cdot(\rho_t v_t)=0.
 \label{eq:continuity}
\end{equation}
The velocity \(v_t\) belongs to the \(L^2(\rho_t)\)-closure of gradient vector fields; in the smooth setting, we may write \(v_t(x)=\nabla_x\phi_t(x)\). Its intrinsic acceleration is the material derivative of the velocity field,
\begin{align}
 a_t(x)
 &=\partial_t v_t(x)+D_xv_t(x)\,v_t(x) \label{eq:covacc}\\
 &=\nabla_x\!\left(\partial_t\phi_t(x)+\tfrac12\|\nabla_x\phi_t(x)\|^2\right).
\end{align}
Thus, the spatial derivative is taken in the direction \(v_t(x)\) of the flow itself~\citep{otto2001geometry,justiniano2024approximation}. The corresponding acceleration energy
\[
\mathcal A(\rho)=\int_0^1\|a_t\|_{\rho_t}^2\,dt,
\qquad
\|h\|_{\rho_t}^2=\int\|h(x)\|^2\,d\rho_t(x),
\]
depends only on the full marginal curve \((\rho_t)_{t\in[0,1]}\).

\paragraph{Identical marginals do not imply identical paths.}
Consider a random path $X_t$ with law $\rho_t$. The conditional velocity $\E[\dot X_t\mid X_t=x]$ satisfies \cref{eq:continuity}, but may include motion that leaves the distribution unchanged. For $Z\sim\mathcal N(0,I_2)$, compare $X_t^{\rm stat}=Z$ and $X_t^{\rm rot}=R(\omega t)Z$, where $R(\alpha)$ rotates the plane by angle $\alpha$. Both have law $\mathcal N(0,I_2)$ at \emph{every} time because Gaussian is rotationally symmetric. Yet $\ddot X_t^{\rm stat}=0$ and $\ddot X_t^{\rm rot}=-\omega^2X_t^{\rm rot}$, giving
\begin{equation}
 \E\|\ddot X_t^{\rm stat}\|^2=0,\qquad
 \E\|\ddot X_t^{\rm rot}\|^2=2\omega^4.
 \label{eq:rotwitness}
\end{equation}
The common marginal curve is stationary, so $v_t=a_t=0$. Thus, the nonzero pathwise acceleration energy measures rotation, not population change.

\begin{figure}[t]
\centering
\includegraphics[width=0.78\linewidth]{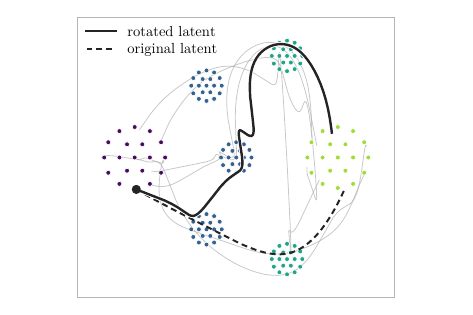}
\caption{\textbf{Same population, different particle paths.} Rotating the latent changes particle paths (solid versus dashed) but preserves every marginal. Dots show the observed snapshots; gray lines show other rotated paths. Particle acceleration thus depends on the representation, whereas the population acceleration $a_M$ is unchanged for a fixed dictionary.}
\label{fig:rotation}
\end{figure}

More generally, for $Z\sim\mathcal N(0,I)$, the paths $G_t(Z)$ and $G_t(R_tZ)$ share every marginal for any time-dependent generator $G_t$ and orthogonal $R_t$ (\cref{fig:rotation}). A prior on the population should therefore give all such representations the same penalty, even when the common marginal curve evolves.

\section{Neural Wasserstein Splines (NWSP)}
\label{sec:method}\label{sec:nnmethod}
NWSP models population dynamics using smooth generative particle paths. Their analytic derivatives are used to fit a population-level velocity field, which is then differentiated to approximate the intrinsic population acceleration. The generator is trained to match the observed snapshots while penalizing this acceleration; NWSP-N extends the method by augmenting the fixed velocity basis with learned potential gradients. \Cref{fig:overview} gives an overview.

\subsection{A spline-based generative model}
\label{sec:generator}
We first describe a generative model for random particle paths \(X_\theta(t,Z)\), where \(Z\sim\mathcal N(0,I_{d_z})\) is a latent draw and \(\theta\) denotes neural network parameters. We write \(X_t\) when the dependence on \(\theta\) and \(Z\) is clear.

More specifically, a neural network $C_\theta:\R^{d_z}\to\R^{J\times d}$ maps $Z$ to $J=|\mathcal R|+2$ control points, which are vectors in $\mathbb{R}^d$ that shape the generated particle trajectory. With knots at the times indexed by $\mathcal{R}$ and clamped cubic B-spline basis functions $B_{j,3}$ mapping $t$ to $\mathbb{R}$,
\begin{equation}
 X_t=\sum_{j=1}^J B_{j,3}(t)C_\theta(Z)_j,
 \qquad \rho_t=\operatorname{Law}(X_t).
 \label{eq:gen}
\end{equation}
A single fixed latent draw \(Z=z\) determines the entire deterministic particle path \(t\mapsto X_\theta(t,z)\), while randomness over \(Z\) induces the marginal distribution \(\rho_t=\operatorname{Law}(X_\theta(t,Z))\). Simple interior knots give $C^2$ paths with analytic derivatives
\begin{equation}
 \begin{aligned}
 U_t:=\dot X_t&=\sum_j\dot B_{j,3}(t)C_\theta(Z)_j,\\
 W_t:=\ddot X_t&=\sum_j\ddot B_{j,3}(t)C_\theta(Z)_j.
 \end{aligned}
 \label{eq:UW}
\end{equation}
A single forward pass of \(C_\theta\) thus gives particle positions, velocities, and accelerations analytically, without velocity measurements or numerical integration. %

\subsection{From particles to population velocity}
\label{sec:projection}

Next, we construct a population-level velocity vector field from the positions and velocities generated by the spline model. %
Let $f_1,\ldots,f_M: \mathbb{R}^d \mapsto \mathbb{R}$ be smooth potentials and their gradients
$F(x)=[\nabla f_1(x)\cdots\nabla f_M(x)]\in\R^{d\times M}$.
We approximate the particle velocity $U_t$ within a span of finite basis:
\begin{equation} \label{eq:vm}
   v_M(t,x)=F(x)c(t) \text{ with coefficients } c(t)\in\R^M.
\end{equation} %
We solve this as a regression problem. Suppressing time arguments, define matrices
\begin{equation}
 A=\E[F(X_t)^\top F(X_t)],\qquad
 b=\E[F(X_t)^\top U_t].
 \label{eq:moments}
\end{equation}
For ridge parameter $\lambda>0$, the fitted
\begin{equation}
 c=\argmin_{q\in\R^M}\{\E\|F(X_t)q-U_t\|^2+\lambda\|q\|^2\}
 \label{eq:am_objective}
\end{equation}
satisfies
\begin{equation}
 (A+\lambda I)c=b.
 \label{eq:solve_c}
\end{equation}
\paragraph{Relation to Action Matching.}
The continuity equation \eqref{eq:continuity} has a weak form: for every smooth test function $f$,
\[
 \frac{d}{dt}\int f\,d\rho_t=\int\nabla f\cdot v_t\,d\rho_t,
\]
so the velocity can be probed through test functions without solving the PDE. Action Matching (AM) \citep{neklyudov2023action} turns this identity into a learning problem: it fits a neural potential $s_t(x)$, with $v_t=\nabla s_t$, to samples of $\rho_t$ at many times, and since samples carry no velocities its objective must differentiate the learned potential in time.
Our setting differs in two ways. The generator supplies particle paths, so the left side is available in closed form: for $f=f_j$ the chain rule gives $\E[\nabla f_j(X_t)\cdot U_t]=b_j$, with no time derivative of any learned function. And the potentials are a fixed finite family, so the fit is the ridge regression above, AM's least-squares objective restricted to $\operatorname{span}\{\nabla f_j\}$ and solved at each $t$. Regressing $v_t$ instead of $U_t$ gives the same normal equations, because the non-gradient part of $\E[U_t\mid X_t]$ is orthogonal to the span. The velocity estimate is thus not new; the new step is differentiating its solution in time along the moving population (\cref{sec:moving}).
At finite $M$ and nonzero ridge, $v_M$ need not solve the continuity equation exactly. NWSP uses $M_0$ fixed linear, quadratic, and sinusoidal potentials, with $M=M_0$, and estimates expectations by minibatches.

\subsection{From velocity to population acceleration}
\label{sec:moving}
The fitted population velocity \(v_M(t,x)\) changes with time because the regression moments \(A(t)\) and \(b(t)\) are evaluated along the time-varying particle paths \(X_t\). To differentiate these moments, let \(D_xF(x)[u]\) denote the directional derivative of \(F\) at \(x\) in direction \(u\). Since \(\dot X_t=U_t\), the chain rule gives
\[
\dot F_t:=\frac{d}{dt}F(X_t)=D_xF(X_t)[U_t].
\] With $F=F(X_t)$ inside expectations,
\begin{equation}
 \begin{aligned}
 \dot A&=\E[\dot F^\top F+F^\top\dot F],\\
 \dot b&=\E[\dot F^\top U_t+F^\top W_t].
 \end{aligned}
 \label{eq:dmoments}
\end{equation}
Differentiating \eqref{eq:solve_c} at fixed $(f_i)_{i=1}^M$ and $\lambda$ yields
\begin{equation}
 (A+\lambda I)\dot c=\dot b-\dot A c.
 \label{eq:solve_cdot}
\end{equation}
The correction \(-\dot A c\) arises because the regression Gram matrix \(A(t)\) changes with time. Taking material derivative of \eqref{eq:vm} yields the following approximation to the intrinsic acceleration field \eqref{eq:covacc}:
\begin{equation}
 a_M(t,x)=F(x)\dot c+D_xF(x)[F(x)c]c.
 \label{eq:accel}
\end{equation}
The population acceleration is therefore obtained by differentiating the fitted velocity field \(v_M\), not by projecting the particle acceleration \(W_t\) as \(U_t\) was projected to obtain \(v_M\): in the rotating-Gaussian example of Section~\ref{sec:problem}, \(W_t\) is nonzero although the marginal velocity and intrinsic acceleration vanish.

\subsection{Training objective and sampling}
\label{sec:objective}
Write the approximated acceleration energy as $\mathcal E(\theta;F)=\int_0^1\|a_M(t,\cdot)\|_{\rho_t}^2dt$. Set training objective
\begin{equation}
\begin{aligned}
 \mathcal L_{\rm fit}(\theta)&=\frac1{|\mathcal R|}\sum_{k\in\mathcal R}
 \left[S_\varepsilon(\rho_{t_k},\hat\rho_{t_k})
       +\gamma\MMD^2(\rho_{t_k},\hat\rho_{t_k})\right],\\
 \mathcal L(\theta;F)&=\mathcal L_{\rm fit}(\theta)+\lambda_a\mathcal E(\theta;F)/L^2.
\end{aligned}
\label{eq:loss}
\end{equation}
Here $S_\varepsilon$ is the debiased Sinkhorn divergence with entropic hyperparameter $\varepsilon$ \citep{cuturi2013sinkhorn,feydy2019interpolating}, and MMD uses Gaussian kernels \citep{gretton2012kernel}. $L^2$ is a fixed squared root-mean-square radius of the pooled retained data. Training first fits the observed snapshots, after which the acceleration-regularization weight $\lambda_a^{(\ell)}$ is gradually increased from zero to $\lambda_a$. See \Cref{alg:step}. %

At each sampled time \(t\in\mathcal T\), we draw two independent latent minibatches \((Z_i)_{i=1}^B\) and \((\widetilde Z_i)_{i=1}^B\). The first batch generates the particles used to estimate \(\widehat A,\widehat b,\widehat{\dot A},\widehat{\dot b}\); the second generates evaluation particles \(\widetilde X_{t,i}=X_\theta(t,\widetilde Z_i)\), on which the acceleration energy is computed:
\begin{equation}
 \widehat{\mathcal E}=\frac1{|\mathcal T|B}
 \sum_{t\in\mathcal T}\sum_{i=1}^B
 \|\hat a_M(t,\widetilde X_{t,i})\|^2.
 \label{eq:energy_estimate}
\end{equation}
Optimization details, schedules, normalization, and batch sizes are in \cref{sec:nnimpl}; computational cost is discussed in \cref{sec:cost}.

\begin{algorithm}[t]
\caption{NWSP and NWSP-N training ($\lambda_a>0$)}
\label{alg:step}\label{alg:nnstep}
\footnotesize
\begin{algorithmic}[1]
\Require Retained snapshots, $C_\theta$, dictionary $F$, $\lambda$, $\lambda_a$, $L^2$, batch size $B$
\State Warm up $C_\theta$; if $K>0$, pretrain $g_\psi$ using $\widehat{\mathcal Q}$
\For{each subsequent optimization step $\ell$}
  \State Sample times $\mathcal T$ and snapshot minibatches
  \State With $\theta$ fixed, update $\psi$ to maximize $\widehat{\mathcal Q}$ if $K>0$
  \State Fix $F$ and its scales; compute minibatch $\widehat{\mathcal L}_{\rm fit}$
  \For{each $t\in\mathcal T$}
    \State Draw projection and evaluation batches
    \State Estimate $\widehat A,\widehat b,\widehat{\dot A},\widehat{\dot b}$ using \cref{eq:moments,eq:dmoments}
    \State Factor $\widehat H=\widehat A+\lambda I$; solve $\widehat H\hat c=\widehat b$
    \State Reuse the factorization: $\widehat H\widehat{\dot c}=\widehat{\dot b}-\widehat{\dot A}\hat c$
    \State Evaluate $\hat a_M$ on $\widetilde X_{t,i}$ using \cref{eq:accel}
  \EndFor
  \State Form $\widehat{\mathcal E}$ using \cref{eq:energy_estimate}
  \State Update $\theta$ to minimize $\widehat{\mathcal L}_{\rm fit}+\lambda_a^{(\ell)}\widehat{\mathcal E}/L^2$, differentiating through both solves
\EndFor
\end{algorithmic}
\end{algorithm}

\subsection{NWSP-N: Expanding the basis}
\label{sec:nncritic}\label{sec:nndict}
We propose an extension of NWSP that augments the fixed potential basis in \cref{sec:projection} with learned potentials. NWSP-N adds $K$ scalar potentials from a neural network $g_\psi:\R^d\to\R^K$:
\begin{equation}
 F_\psi(x)=\big[\nabla f_1\cdots\nabla f_{M_0}\mid
                 \nabla g_{\psi,1}\cdots\nabla g_{\psi,K}\big].
 \label{eq:nndict}
\end{equation}

With \(\theta\) fixed, NWSP-N learns the added potentials so that their gradients capture as much of the population velocity field as possible. At each time \(t\), define the regularized explained velocity energy
\begin{equation}
 \begin{aligned}
 \Jcap_r(\psi,t)
 &=b^\top(A+\lambda I)^{-1}b\\
 &=\|v_t\|_{\rho_t}^2
 -\min_{q\in\mathbb R^M}
 \left\{
 \|v_t-F_\psi q\|_{\rho_t}^2+\lambda\|q\|^2
 \right\}.
 \end{aligned}
 \label{eq:critic}
\end{equation}
This measures how much of the population velocity energy can be represented by $F_\psi$.
We learn \(\psi\) by maximizing this quantity across time. To prevent times with larger particle speeds from dominating the objective, we normalize by the spline-particle velocity scale,
\begin{equation}
 \mathcal Q(\psi;\theta)
 =
 \int_0^1
 \frac{\Jcap_r(\psi,t)}
 {\max(\E\|U_t\|^2,\epsilon_U)}
 \,dt,
 \label{eq:potential_learning}
\end{equation}
The two-independent-batch estimator \(\widehat{\mathcal Q}\) used in practice is given in \cref{sec:nnimpl}. The learned gradient features are additionally normalized by their root-mean-square magnitude to control feature scaling.

NWSP-N alternates between updating \(g_\psi\) and updating \(C_\theta\); see \cref{alg:nnstep}. Potential updates hold \(\theta\) fixed and improve the feature basis so that it captures the population velocity more accurately. Generator updates instead hold \(\psi\) and the feature-normalization scales fixed, and update \(\theta\) to fit the data. We do not minimize the acceleration penalty with respect to \(\psi\), since doing so could favor features that simply fail to represent the underlying motion and therefore make the measured acceleration artificially small.

\subsection{Theoretical guarantees}
\label{sec:theory}\label{sec:nntheory}
The surrogate is useful as a population prior only if it does not depend on the particle representation, is stable to estimation error, and can resolve the motion of interest; we address these three points in turn. Statements and proofs are in \cref{sec:proofs}. They rely on a regularity condition (R) stated there: twice-differentiable paths, thrice-differentiable potentials, and integrable bounds that allow time derivatives inside expectations. No density is assumed; the Fourier approximation result additionally assumes the periodic setting and positive densities stated there.

\paragraph{Does the penalty depend on how the particles are generated?} Not for a fixed dictionary.
\begin{proposition}[Representation invariance]
\label[proposition]{prop:invariance}
Fix a time-independent dictionary and ridge parameter $\lambda>0$. Two lifts satisfying (R) and inducing the same marginal curve yield identical $A,b,\dot A,\dot b$, hence identical $v_M,a_M$ and, whenever finite, population acceleration energy.
\end{proposition}
The reason is that $A_{ij}=\int\nabla f_i\cdot\nabla f_j\,d\rho_t$ and $b_j=\partial_t\int f_j\,d\rho_t$, and hence their time derivatives, are averages over the population and never refer to which particle is which. The same holds for any fixed learned dictionary. It does not extend to the full NWSP-N procedure, because the potential-learning objective \eqref{eq:potential_learning} is normalized by the path-dependent particle velocity energy $\E\|U_t\|^2$, and it is an exact statement about populations, not about finite batches.

\paragraph{How accurate is the estimate, and why differentiate the projection?}
\Cref{thm:stability} bounds the error of $a_M$ by the errors in all four moments $A,b,\dot A,\dot b$, not only those of the velocity fit, because the acceleration error involves time and space derivatives of the velocity error (\cref{eq:accdecomp} in \cref{sec:proofs}). For a regular reference curve with exact moments (Sobolev regularity $s>1+d/2$, positive densities, ridge $\lambda_N=O(N^{-s})$), a Fourier dictionary with cutoff $N$ has velocity error $O(N^{-s})$ and acceleration error $O(N^{1-s})$ (\cref{thm:consistency}); these rates do not cover the trained generator or a learned dictionary. Differentiating the projection is also necessary: dropping $-\dot A c$ in \cref{eq:solve_cdot} adds $F(A+\lambda I)^{-1}\dot A c$ to the acceleration, which converges to $\Pi_{\rho_t}\big[(\partial_t\rho_t/\rho_t)v_t\big]$, where $\Pi_{\rho_t}$ is the weighted projection onto gradient fields (\cref{cor:gram_bias}). This field need not vanish, so enlarging the dictionary does not repair the omission.

\paragraph{What can a finite dictionary resolve?}
Invariance does not imply adequate resolution. Linear and quadratic potentials cannot see shape changes that preserve the mean and variance; \cref{sec:mixtureproof} gives a Gaussian-mixture example in which they report zero acceleration for a changing distribution. Learning the potentials addresses this limitation of the projection space; we assess it empirically in \cref{sec:adaptive}, as the theory does not guarantee that the learning rule resolves every such change.

\begin{table*}[!t]
\centering
\footnotesize
\caption{\textbf{Held-out $W_1$} (lower is better): mean over three seeds, with one sample SD below. WLR and 4-point are deterministic (no spread); their subsampled entries average three draws. Bold and underline mark the best and second-best means, compared before rounding; the last rows count best and top-two results. NWSP and NWSP-N are ours. Withheld snapshots are scored at full support (\cref{sec:scalability}).}
\label{tab:main}
\setlength{\tabcolsep}{3.5pt}
\renewcommand{\arraystretch}{0.92}
\nnfit{%
\begin{tabular}{@{}lrrrrrrrrr@{}}
\toprule
Dataset & WLR & 4-point & OT-CFM & SB-CFM & CFM & {[SF]}$^{2}$M & MFM & \ourhead{NWSP} & \ourhead{NWSP-N} \\
\midrule
\multicolumn{10}{@{}l}{\emph{Leave-one-out}} \\
\shortstack[l]{Diverging Gauss.\\{\scriptsize\vphantom{$\pm 0$}}} & \shortstack[r]{$1.194$\\{\scriptsize \vphantom{$\pm 0$}}} & \shortstack[r]{$1.193$\\{\scriptsize \vphantom{$\pm 0$}}} & \shortstack[r]{$1.173$\\{\scriptsize $\pm\,0.035$}} & \shortstack[r]{$1.192$\\{\scriptsize $\pm\,0.038$}} & \shortstack[r]{$1.275$\\{\scriptsize $\pm\,0.022$}} & \shortstack[r]{$1.164$\\{\scriptsize $\pm\,0.014$}} & \shortstack[r]{$1.299$\\{\scriptsize $\pm\,0.116$}} & \shortstack[r]{$\underline{1.136}$\\{\scriptsize $\pm\,0.017$}} & \shortstack[r]{$\mathbf{1.117}$\\{\scriptsize $\pm\,0.026$}}  \\
\addlinespace[0pt]
\shortstack[l]{Petal\\{\scriptsize\vphantom{$\pm 0$}}} & \shortstack[r]{$\underline{0.055}$\\{\scriptsize \vphantom{$\pm 0$}}} & \shortstack[r]{$\mathbf{0.052}$\\{\scriptsize \vphantom{$\pm 0$}}} & \shortstack[r]{$0.063$\\{\scriptsize $\pm\,0.003$}} & \shortstack[r]{$0.064$\\{\scriptsize $\pm\,0.002$}} & \shortstack[r]{$0.174$\\{\scriptsize $\pm\,0.037$}} & \shortstack[r]{$0.092$\\{\scriptsize $\pm\,0.004$}} & \shortstack[r]{$0.084$\\{\scriptsize $\pm\,0.005$}} & \shortstack[r]{$0.090$\\{\scriptsize $\pm\,0.014$}} & \shortstack[r]{$0.085$\\{\scriptsize $\pm\,0.009$}}  \\
\addlinespace[0pt]
\shortstack[l]{Converging Gauss.\\{\scriptsize\vphantom{$\pm 0$}}} & \shortstack[r]{$0.453$\\{\scriptsize \vphantom{$\pm 0$}}} & \shortstack[r]{$0.438$\\{\scriptsize \vphantom{$\pm 0$}}} & \shortstack[r]{$0.443$\\{\scriptsize $\pm\,0.012$}} & \shortstack[r]{$\underline{0.434}$\\{\scriptsize $\pm\,0.012$}} & \shortstack[r]{$0.483$\\{\scriptsize $\pm\,0.006$}} & \shortstack[r]{$\mathbf{0.431}$\\{\scriptsize $\pm\,0.012$}} & \shortstack[r]{$0.435$\\{\scriptsize $\pm\,0.044$}} & \shortstack[r]{$0.452$\\{\scriptsize $\pm\,0.022$}} & \shortstack[r]{$0.469$\\{\scriptsize $\pm\,0.025$}}  \\
\addlinespace[0pt]
\shortstack[l]{Embryoid body\\{\scriptsize\vphantom{$\pm 0$}}} & \shortstack[r]{$0.428$\\{\scriptsize \vphantom{$\pm 0$}}} & \shortstack[r]{$0.407$\\{\scriptsize \vphantom{$\pm 0$}}} & \shortstack[r]{$0.472$\\{\scriptsize $\pm\,0.023$}} & \shortstack[r]{$0.464$\\{\scriptsize $\pm\,0.022$}} & \shortstack[r]{$0.495$\\{\scriptsize $\pm\,0.019$}} & \shortstack[r]{$0.464$\\{\scriptsize $\pm\,0.015$}} & \shortstack[r]{$0.397$\\{\scriptsize $\pm\,0.025$}} & \shortstack[r]{$\mathbf{0.350}$\\{\scriptsize $\pm\,0.007$}} & \shortstack[r]{$\underline{0.351}$\\{\scriptsize $\pm\,0.003$}}  \\
\addlinespace[0pt]
\shortstack[l]{CITE-seq\\{\scriptsize\vphantom{$\pm 0$}}} & \shortstack[r]{$31.062$\\{\scriptsize \vphantom{$\pm 0$}}} & \shortstack[r]{$31.688$\\{\scriptsize \vphantom{$\pm 0$}}} & \shortstack[r]{$32.422$\\{\scriptsize $\pm\,0.102$}} & \shortstack[r]{$32.252$\\{\scriptsize $\pm\,0.150$}} & \shortstack[r]{$33.246$\\{\scriptsize $\pm\,0.536$}} & \shortstack[r]{$33.960$\\{\scriptsize $\pm\,0.223$}} & \shortstack[r]{$31.465$\\{\scriptsize $\pm\,0.115$}} & \shortstack[r]{$\underline{30.545}$\\{\scriptsize $\pm\,0.106$}} & \shortstack[r]{$\mathbf{30.527}$\\{\scriptsize $\pm\,0.093$}}  \\
\addlinespace[0pt]
\shortstack[l]{Circular Gauss.\\{\scriptsize\vphantom{$\pm 0$}}} & \shortstack[r]{$1.961$\\{\scriptsize \vphantom{$\pm 0$}}} & \shortstack[r]{$1.406$\\{\scriptsize \vphantom{$\pm 0$}}} & \shortstack[r]{$1.545$\\{\scriptsize $\pm\,0.067$}} & \shortstack[r]{$1.579$\\{\scriptsize $\pm\,0.023$}} & \shortstack[r]{$1.560$\\{\scriptsize $\pm\,0.052$}} & \shortstack[r]{$1.592$\\{\scriptsize $\pm\,0.048$}} & \shortstack[r]{$0.886$\\{\scriptsize $\pm\,0.119$}} & \shortstack[r]{$\underline{0.427}$\\{\scriptsize $\pm\,0.015$}} & \shortstack[r]{$\mathbf{0.426}$\\{\scriptsize $\pm\,0.014$}}  \\
\addlinespace[0pt]
\shortstack[l]{Dyngen Tree\\{\scriptsize\vphantom{$\pm 0$}}} & \shortstack[r]{$2.454$\\{\scriptsize \vphantom{$\pm 0$}}} & \shortstack[r]{$2.444$\\{\scriptsize \vphantom{$\pm 0$}}} & \shortstack[r]{$2.489$\\{\scriptsize $\pm\,0.056$}} & \shortstack[r]{$2.516$\\{\scriptsize $\pm\,0.112$}} & \shortstack[r]{$\underline{2.391}$\\{\scriptsize $\pm\,0.088$}} & \shortstack[r]{$2.481$\\{\scriptsize $\pm\,0.046$}} & \shortstack[r]{$2.517$\\{\scriptsize $\pm\,0.054$}} & \shortstack[r]{$\mathbf{2.388}$\\{\scriptsize $\pm\,0.031$}} & \shortstack[r]{$2.409$\\{\scriptsize $\pm\,0.006$}}  \\
\addlinespace[0pt]
\shortstack[l]{Dyngen Cycle\\{\scriptsize\vphantom{$\pm 0$}}} & \shortstack[r]{$0.878$\\{\scriptsize \vphantom{$\pm 0$}}} & \shortstack[r]{$0.758$\\{\scriptsize \vphantom{$\pm 0$}}} & \shortstack[r]{$0.685$\\{\scriptsize $\pm\,0.006$}} & \shortstack[r]{$0.690$\\{\scriptsize $\pm\,0.012$}} & \shortstack[r]{$0.734$\\{\scriptsize $\pm\,0.010$}} & \shortstack[r]{$0.688$\\{\scriptsize $\pm\,0.023$}} & \shortstack[r]{$0.676$\\{\scriptsize $\pm\,0.008$}} & \shortstack[r]{$\mathbf{0.607}$\\{\scriptsize $\pm\,0.012$}} & \shortstack[r]{$\underline{0.609}$\\{\scriptsize $\pm\,0.011$}}  \\
\addlinespace[0pt]
\shortstack[l]{Mouse hemat.\\{\scriptsize\vphantom{$\pm 0$}}} & \shortstack[r]{$0.406$\\{\scriptsize \vphantom{$\pm 0$}}} & \shortstack[r]{$0.402$\\{\scriptsize \vphantom{$\pm 0$}}} & \shortstack[r]{$0.412$\\{\scriptsize $\pm\,0.009$}} & \shortstack[r]{$0.413$\\{\scriptsize $\pm\,0.022$}} & \shortstack[r]{$0.524$\\{\scriptsize $\pm\,0.012$}} & \shortstack[r]{$0.401$\\{\scriptsize $\pm\,0.027$}} & \shortstack[r]{$\mathbf{0.363}$\\{\scriptsize $\pm\,0.009$}} & \shortstack[r]{$\underline{0.399}$\\{\scriptsize $\pm\,0.007$}} & \shortstack[r]{$0.401$\\{\scriptsize $\pm\,0.010$}}  \\
\addlinespace[0pt]
\shortstack[l]{Veres $\beta$-cell\\{\scriptsize\vphantom{$\pm 0$}}} & \shortstack[r]{$8.191$\\{\scriptsize \vphantom{$\pm 0$}}} & \shortstack[r]{$8.353$\\{\scriptsize \vphantom{$\pm 0$}}} & \shortstack[r]{$8.414$\\{\scriptsize $\pm\,0.051$}} & \shortstack[r]{$8.535$\\{\scriptsize $\pm\,0.013$}} & \shortstack[r]{$8.873$\\{\scriptsize $\pm\,0.022$}} & \shortstack[r]{$8.461$\\{\scriptsize $\pm\,0.044$}} & \shortstack[r]{$8.121$\\{\scriptsize $\pm\,0.025$}} & \shortstack[r]{$\mathbf{7.935}$\\{\scriptsize $\pm\,0.008$}} & \shortstack[r]{$\underline{7.942}$\\{\scriptsize $\pm\,0.014$}}  \\
\addlinespace[0pt]
\midrule
\multicolumn{10}{@{}l}{\emph{Leave-alternate-out}} \\
\shortstack[l]{Petal\\{\scriptsize\vphantom{$\pm 0$}}} & \shortstack[r]{$\mathbf{0.040}$\\{\scriptsize \vphantom{$\pm 0$}}} & \shortstack[r]{$\underline{0.045}$\\{\scriptsize \vphantom{$\pm 0$}}} & \shortstack[r]{$0.120$\\{\scriptsize $\pm\,0.004$}} & \shortstack[r]{$0.064$\\{\scriptsize $\pm\,0.009$}} & \shortstack[r]{$0.202$\\{\scriptsize $\pm\,0.010$}} & \shortstack[r]{$0.086$\\{\scriptsize $\pm\,0.016$}} & \shortstack[r]{$0.064$\\{\scriptsize $\pm\,0.005$}} & \shortstack[r]{$0.086$\\{\scriptsize $\pm\,0.011$}} & \shortstack[r]{$0.079$\\{\scriptsize $\pm\,0.017$}}  \\
\addlinespace[0pt]
\shortstack[l]{Embryoid body\\{\scriptsize\vphantom{$\pm 0$}}} & \shortstack[r]{$\mathbf{0.404}$\\{\scriptsize \vphantom{$\pm 0$}}} & \shortstack[r]{$0.441$\\{\scriptsize \vphantom{$\pm 0$}}} & \shortstack[r]{$0.455$\\{\scriptsize $\pm\,0.018$}} & \shortstack[r]{$0.465$\\{\scriptsize $\pm\,0.033$}} & \shortstack[r]{$0.479$\\{\scriptsize $\pm\,0.047$}} & \shortstack[r]{$0.434$\\{\scriptsize $\pm\,0.014$}} & \shortstack[r]{$0.463$\\{\scriptsize $\pm\,0.028$}} & \shortstack[r]{$0.412$\\{\scriptsize $\pm\,0.007$}} & \shortstack[r]{$\underline{0.410}$\\{\scriptsize $\pm\,0.004$}}  \\
\addlinespace[0pt]
\shortstack[l]{Circular Gauss.\\{\scriptsize\vphantom{$\pm 0$}}} & \shortstack[r]{$1.908$\\{\scriptsize \vphantom{$\pm 0$}}} & \shortstack[r]{$0.780$\\{\scriptsize \vphantom{$\pm 0$}}} & \shortstack[r]{$1.665$\\{\scriptsize $\pm\,0.118$}} & \shortstack[r]{$1.663$\\{\scriptsize $\pm\,0.129$}} & \shortstack[r]{$1.677$\\{\scriptsize $\pm\,0.106$}} & \shortstack[r]{$1.716$\\{\scriptsize $\pm\,0.081$}} & \shortstack[r]{$1.587$\\{\scriptsize $\pm\,0.422$}} & \shortstack[r]{$\underline{0.484}$\\{\scriptsize $\pm\,0.007$}} & \shortstack[r]{$\mathbf{0.479}$\\{\scriptsize $\pm\,0.014$}}  \\
\addlinespace[0pt]
\shortstack[l]{Dyngen Tree\\{\scriptsize\vphantom{$\pm 0$}}} & \shortstack[r]{$2.635$\\{\scriptsize \vphantom{$\pm 0$}}} & \shortstack[r]{$2.678$\\{\scriptsize \vphantom{$\pm 0$}}} & \shortstack[r]{$2.683$\\{\scriptsize $\pm\,0.014$}} & \shortstack[r]{$2.715$\\{\scriptsize $\pm\,0.033$}} & \shortstack[r]{$\mathbf{2.561}$\\{\scriptsize $\pm\,0.047$}} & \shortstack[r]{$2.641$\\{\scriptsize $\pm\,0.030$}} & \shortstack[r]{$\underline{2.582}$\\{\scriptsize $\pm\,0.018$}} & \shortstack[r]{$2.624$\\{\scriptsize $\pm\,0.011$}} & \shortstack[r]{$2.657$\\{\scriptsize $\pm\,0.021$}}  \\
\addlinespace[0pt]
\shortstack[l]{Dyngen Cycle\\{\scriptsize\vphantom{$\pm 0$}}} & \shortstack[r]{$0.859$\\{\scriptsize \vphantom{$\pm 0$}}} & \shortstack[r]{$0.750$\\{\scriptsize \vphantom{$\pm 0$}}} & \shortstack[r]{$0.830$\\{\scriptsize $\pm\,0.034$}} & \shortstack[r]{$0.851$\\{\scriptsize $\pm\,0.028$}} & \shortstack[r]{$0.839$\\{\scriptsize $\pm\,0.018$}} & \shortstack[r]{$0.832$\\{\scriptsize $\pm\,0.025$}} & \shortstack[r]{$0.802$\\{\scriptsize $\pm\,0.019$}} & \shortstack[r]{$\mathbf{0.738}$\\{\scriptsize $\pm\,0.019$}} & \shortstack[r]{$\underline{0.738}$\\{\scriptsize $\pm\,0.022$}}  \\
\addlinespace[0pt]
\shortstack[l]{Veres $\beta$-cell\\{\scriptsize\vphantom{$\pm 0$}}} & \shortstack[r]{$8.385$\\{\scriptsize \vphantom{$\pm 0$}}} & \shortstack[r]{$8.560$\\{\scriptsize \vphantom{$\pm 0$}}} & \shortstack[r]{$8.553$\\{\scriptsize $\pm\,0.098$}} & \shortstack[r]{$8.621$\\{\scriptsize $\pm\,0.181$}} & \shortstack[r]{$8.924$\\{\scriptsize $\pm\,0.105$}} & \shortstack[r]{$8.748$\\{\scriptsize $\pm\,0.395$}} & \shortstack[r]{$\mathbf{7.973}$\\{\scriptsize $\pm\,0.058$}} & \shortstack[r]{$\underline{8.153}$\\{\scriptsize $\pm\,0.029$}} & \shortstack[r]{$8.160$\\{\scriptsize $\pm\,0.026$}}  \\
\addlinespace[0pt]
\midrule
Best / 16 & 2 & 1 & 0 & 0 & 1 & 1 & 2 & $\mathbf{5}$ & $\underline{4}$ \\
Top two / 16 & 3 & 2 & 0 & 1 & 2 & 1 & 3 & $\mathbf{11}$ & $\underline{9}$ \\
\bottomrule
\end{tabular}}
\end{table*}

\section{Experiments}
\label{sec:experiments}
\subsection{Experimental setup}
\label{sec:v2protocol}
\paragraph{Data and prediction tasks.}
Ten datasets provide sixteen dataset--protocol settings and 66 held-out snapshots in dimensions $2$--$100$ (\cref{tab:nn_datasets}). They include synthetic trajectories, dyngen simulations \citep{cannoodt2021spearheading}, embryoid body \citep{moon2019visualizing}, Mouse hematopoiesis \citep{weinreb2020lineage}, Veres $\beta$-cell data \citep{veres2019charting}, and CITE-seq \citep{burkhardt2022openproblems,banerjee2025efficient}, following the benchmark sources in \cref{sec:v2app}. Leave-one-out removes each interior snapshot in turn; leave-alternate-out removes every second one jointly.

\paragraph{Baselines and metrics.}
We compare NWSP and NWSP-N with WLR, four-point subdivision, CFM, OT-CFM, SB-CFM, $[\mathrm{SF}]^2$M, and MFM. The primary metric is the exact empirical $W_1$ between the predicted population and the withheld snapshot; we also report $W_2$ and Gaussian-kernel MMD$^2$ (\cref{eq:wp,eq:mmd}). Scores average over holdouts within each seed and then over three seeds, with one sample standard deviation as spread. All metrics use the same models, selected on $W_1$ with the target score excluded (\cref{sec:selection}).

\begin{table*}[!t]
\centering
\small
\caption{\textbf{Effect of the acceleration prior.} Held-out $W_1$ (mean $\pm$ SD over three seeds) for the same generator trained without (Spline-only) or with the population-acceleration penalty. Bold and underline mark the best and second-best mean within each protocol; a dash marks an inapplicable protocol.}
\label{tab:decomp}
\setlength{\tabcolsep}{10pt}
\renewcommand{\arraystretch}{1.0}
\begin{tabular}{@{}lrrrr@{}}
\toprule
& \multicolumn{2}{c}{Leave-one-out} & \multicolumn{2}{c}{Leave-alternate-out} \\
\cmidrule(lr){2-3}\cmidrule(lr){4-5}
Dataset & \spline{} & \textbf{NWSP} & \spline{} & \textbf{NWSP} \\
\midrule
Diverging Gauss. & $\mathbf{1.084}\pm0.011$ & $\underline{1.136}\pm0.017$ & -- & -- \\
Petal & $\underline{0.153}\pm0.004$ & $\mathbf{0.090}\pm0.014$ & $\underline{0.199}\pm0.009$ & $\mathbf{0.086}\pm0.011$ \\
Converging Gauss. & $\underline{0.466}\pm0.025$ & $\mathbf{0.452}\pm0.022$ & -- & -- \\
Embryoid body & $\underline{0.450}\pm0.003$ & $\mathbf{0.350}\pm0.007$ & $\underline{0.457}\pm0.002$ & $\mathbf{0.412}\pm0.007$ \\
CITE-seq & $\underline{30.586}\pm0.052$ & $\mathbf{30.545}\pm0.106$ & -- & -- \\
Circular Gauss. & $\underline{0.432}\pm0.024$ & $\mathbf{0.427}\pm0.015$ & $\mathbf{0.483}\pm0.010$ & $\underline{0.484}\pm0.007$ \\
Dyngen Tree & $\underline{2.510}\pm0.026$ & $\mathbf{2.388}\pm0.031$ & $\underline{2.628}\pm0.008$ & $\mathbf{2.624}\pm0.011$ \\
Dyngen Cycle & $\underline{0.629}\pm0.004$ & $\mathbf{0.607}\pm0.012$ & $\underline{0.747}\pm0.011$ & $\mathbf{0.738}\pm0.019$ \\
Mouse hemat. & $\underline{0.403}\pm0.017$ & $\mathbf{0.399}\pm0.007$ & -- & -- \\
Veres $\beta$-cell & $\underline{8.258}\pm0.007$ & $\mathbf{7.935}\pm0.008$ & $\mathbf{8.105}\pm0.014$ & $\underline{8.153}\pm0.029$ \\
\bottomrule
\end{tabular}
\end{table*}

\subsection{Prediction between observed snapshots}
\label{sec:v2loo}\label{sec:v2losses}
NWSP attains the lowest $W_1$ in five settings and NWSP-N in four; they place in the top two in eleven and nine settings, respectively (\cref{tab:main}). NWSP improves over OT-CFM, SB-CFM, CFM, and $[\mathrm{SF}]^2$M in $14$, $13$, $15$, and $15$ of the sixteen settings, so direct spline sampling is competitive with learned vector-field models while requiring no trajectory integration at inference.

\paragraph{Curved dynamics and single-cell data.}
On Circular Gaussian, leave-one-out $W_1$ is $0.427$ for NWSP and $0.426$ for NWSP-N, versus $0.886$ for the strongest baseline, MFM; with alternate snapshots removed it is $0.484$ and $0.479$, versus $0.780$ for four-point subdivision. Both variants also lead on Dyngen Cycle under both protocols, and on leave-one-out embryoid body, Veres, and CITE-seq they occupy the first two positions under both $W_1$ and $W_2$ (\cref{tab:full_W2}). For embryoid body, $W_1$ falls from the best baseline's $0.397$ to $0.350$ and $0.351$; on CITE-seq the values are $31.06$, $30.55$, and $30.53$, a smaller gain that persists in the $100$-dimensional evaluation space. MMD$^2$ rankings are less uniform and geometric baselines remain strong: subdivision leads on Petal, MFM on Mouse leave-one-out and Veres leave-alternate-out. Dataset-level summaries and held-out point clouds are in \cref{tab:dataset-tests,fig:sc_heldout_abc,fig:sc_heldout_d}.

\paragraph{Full-snapshot coverage.}\label{sec:scalability}
Both variants draw minibatches from all retained observations, whereas the evaluated exact-transport WLR and four-point implementations require $1000$-point subsamples on embryoid body, Mouse, and Veres (five settings); all methods are scored against complete withheld snapshots. This is a coverage advantage, not a measured speedup (\cref{sec:cost}), and differences in training support prevent attributing every comparison to the acceleration prior, which the next experiment isolates.

\subsection{What does the acceleration prior contribute?}
\label{sec:real}\label{sec:v2whatworks}\label{sec:tuning}
The spline-only model already produces $C^2$ paths. \Cref{tab:decomp} asks whether regularizing population acceleration adds predictive value beyond this smoothness; the two arms share architecture and snapshot loss and differ only in the penalty.

Regularization lowers mean $W_1$ in thirteen of sixteen settings, with median relative change $-2.2\%$. The gains are heterogeneous: embryoid body leave-one-out improves from $0.450$ to $0.350$, and Petal from $0.153$ to $0.090$ and from $0.199$ to $0.086$ under the two protocols, so smooth particle paths alone do not supply the same inductive bias. The CITE-seq change is small, and Diverging Gaussian and Veres leave-alternate-out favor spline-only. The largest gains are concentrated in a few settings and depend on the regularization range (\cref{sec:selection}); Petal remains better served by subdivision.

Replacing projected acceleration with raw particle acceleration gives a closer comparison (\cref{tab:rawmatch}). Projection has lower $W_1$ in eleven settings and higher error in five; the median raw-minus-projected change is $+0.8\%$. With different hardware across the two arms and several small differences, this ablation supports neither a universal predictive advantage nor equivalence of the penalties; their geometric distinction needs a test that changes particle paths while keeping the marginal curve fixed.

\subsection{Does adaptation improve acceleration estimation?}
\label{sec:adaptive}\label{sec:lift}\label{sec:analytic}
\begin{table}[!t]
\centering
\footnotesize
\caption{\textbf{Learning the projection space.} Relative acceleration error on analytic reference flows (mean $\pm$ SD over five seeds; lower is better). Each added block has sixteen potentials; the NWSP-N block is trained for $300$ updates. Bold and underline mark the best and second-best means, ties included.}
\label{tab:ana_e1}
\setlength{\tabcolsep}{3pt}
\renewcommand{\arraystretch}{1.0}
\begin{tabular}{@{}lrrrr@{}}
\toprule
$d$ & \ourhead{NWSP} & \shortstack{+16\\ridge} & \shortstack{+16\\untrained} & \ourhead{NWSP-N} \\
\midrule
2 & \shortstack[r]{$0.705$\\{\scriptsize $\pm\,0.046$}} & \shortstack[r]{$0.683$\\{\scriptsize $\pm\,0.056$}} & \shortstack[r]{$\underline{0.306}$\\{\scriptsize $\pm\,0.025$}} & \shortstack[r]{$\mathbf{0.062}$\\{\scriptsize $\pm\,0.004$}} \\
5 & \shortstack[r]{$0.927$\\{\scriptsize $\pm\,0.025$}} & \shortstack[r]{$0.918$\\{\scriptsize $\pm\,0.023$}} & \shortstack[r]{$\underline{0.899}$\\{\scriptsize $\pm\,0.024$}} & \shortstack[r]{$\mathbf{0.137}$\\{\scriptsize $\pm\,0.016$}} \\
10 & \shortstack[r]{$0.962$\\{\scriptsize $\pm\,0.006$}} & \shortstack[r]{$\underline{0.960}$\\{\scriptsize $\pm\,0.008$}} & \shortstack[r]{$\underline{0.960}$\\{\scriptsize $\pm\,0.006$}} & \shortstack[r]{$\mathbf{0.199}$\\{\scriptsize $\pm\,0.016$}} \\
30 & \shortstack[r]{$\underline{0.975}$\\{\scriptsize $\pm\,0.002$}} & \shortstack[r]{$\underline{0.975}$\\{\scriptsize $\pm\,0.002$}} & \shortstack[r]{$\underline{0.975}$\\{\scriptsize $\pm\,0.002$}} & \shortstack[r]{$\mathbf{0.341}$\\{\scriptsize $\pm\,0.021$}} \\
\bottomrule
\end{tabular}
\end{table}

On analytic reference flows, acceleration is known independently of the fitted dictionary. \Cref{tab:ana_e1} compares the fixed dictionary with three equal-size additions: sixteen ridge features, sixteen untrained neural potentials, or sixteen learned potentials. Projection and error evaluation use independent samples, and results average five seeds; matched sizes isolate learning from dictionary expansion.

After $300$ potential updates, relative acceleration error falls from $0.705$ to $0.062$ at $d=2$ and from $0.975$ to $0.341$ at $d=30$. Learned potentials outperform both matched-size additions at every displayed dimension; at $d=30$ all three non-adaptive dictionaries have error $0.975$, so the improvement is attributable to adaptation rather than feature count. Resolution remains imperfect: error increases with dimension, and adding $64$ learned potentials does not uniformly improve on $16$. Longer optimization can be more useful than a larger block (\cref{tab:ana_budget}).

\paragraph{Testing the differentiated projection.}
A time-dependent Gaussian latent rotation preserves every marginal but raises the raw particle energy from $1.42$ to $1.45\times10^5$ as its rate goes from zero to sixteen, whereas the projected energy varies by at most $2.3\%$ through rate four at the largest sample size (\cref{sec:diagnostics}); exact invariance is a population statement, not a finite-batch guarantee. On the periodic density curve of \cref{sec:periodiccalc}, the moving-Gram correction reduces the absolute $L^2(\rho_t)$ acceleration error from $0.577$ to about $7\times10^{-14}$, and the uncorrected error persists under Fourier refinement: velocity resolution cannot repair an incorrectly differentiated projection.

\section{Discussion and Conclusion}
\label{sec:limitations}\label{sec:conclusion}
NWSP decouples how a population is sampled from how its motion is regularized. The spline generator yields samples and analytic derivatives without cross-time transport, and the differentiated projection charges acceleration to the marginal curve rather than to unidentifiable particle paths, so every representation of the same curve receives the same penalty. Empirically, NWSP is highly competitive with transport- and flow-based baselines, ranking top two in eleven of sixteen settings while training on full snapshots, and its acceleration prior improves on spline smoothness in thirteen of sixteen settings while remaining competitive with raw particle acceleration. Learned potentials sharply improve acceleration estimation on controlled flows (relative error $0.705\to0.062$ at $d=2$), and NWSP-N attains the lowest $W_1$ in four settings, showing that the prior can be adapted within the same framework.

The guarantees show that invariance holds for frozen dictionaries and the Fourier rates assume exact moments and regular positive densities, so extending both to the alternating neural optimization is a natural next step. By design, NWSP identifies marginal curves rather than individual cell histories, and it shifts the bottleneck from full-support transport to dictionary estimation, which reaches $5186$ fixed features at $d=100$ (\cref{sec:cost}). Sparser dictionaries and matched runtime and memory comparisons are promising directions for scaling further.

\clearpage
\nobalance
\subsection*{AI use statement}
Generative AI tools were used to assist with drafting portions of the manuscript and revising the text for clarity, organization, consistency checks and readability. The authors take full responsibility for the final manuscript, including its originality, technical claims, reported results, and references.
\bibliographystyle{apalike}
\bibliography{nwsp}

\clearpage
\appendix
\thispagestyle{empty}
\onecolumn
\aistatstitle{Neural Wasserstein Splines: Representation-Invariant\\
Second-Order Priors for Trajectory Inference:\\
Supplementary Materials}

\section{Implementation and Evaluation}
\label{sec:implementation}\label{sec:v2app}\label{sec:repro}
\expandafter\def\csname compactcriticlr\endcsname{0.003}
\expandafter\def\csname compactcriticlrjoint\endcsname{0.001}
\expandafter\def\csname compactcriticevery\endcsname{1}
\begin{table}[!htbp]
\centering
\footnotesize
\caption{\textbf{Datasets and held-out tasks.} State dimension $d$, snapshots $T$, held-out targets $|\mathcal H|$ under leave-one-out (loo) and leave-alternate-out (alt), and fixed-dictionary size $M_0$; NWSP-N adds $K=16$ potentials. A dash marks an inapplicable protocol.}
\label{tab:nn_datasets}\label{tab:v2_datasets}
\begin{tabular}{lrrrrr}
\toprule
Dataset & $d$ & $T$ & $|\mathcal H|_{\mathrm{loo}}$ & $|\mathcal H|_{\mathrm{alt}}$ & $M_0$ \\
\midrule
Circular Gaussian & 2 & 10 & 8 & 4 & 41 \\
Dyngen Cycle & 10 & 15 & 13 & 7 & 101 \\
Dyngen Tree & 10 & 7 & 5 & 3 & 101 \\
Petal & 2 & 5 & 3 & 2 & 41 \\
Converging Gaussian & 2 & 4 & 2 & -- & 41 \\
Diverging Gaussian & 2 & 4 & 2 & -- & 41 \\
Embryoid body & 2 & 5 & 3 & 2 & 41 \\
Mouse hematopoiesis & 2 & 3 & 1 & -- & 41 \\
Veres $\beta$-cell & 30 & 8 & 6 & 3 & 531 \\
CITE-seq & 100 & 4 & 2 & -- & 5186 \\
\bottomrule
\end{tabular}
\end{table}

 We use $d_z=\min(d,16)$ and an affine component plus a multilayer network, initialized from snapshot means indexed by $\mathcal{R}$ and pooled covariance. Nine datasets follow the WLR, PACE, and BranchSBM benchmarks \citep{banerjee2025efficient,yu2026pace,tang2026branched}; the tenth is the $100$-dimensional CITE-seq time course.

\paragraph{Dataset provenance.}\label{sec:assets}
We load Circular Gaussian, Diverging Gaussian, Petal, Converging Gaussian, Dyngen Tree, Dyngen Cycle, and CITE-seq through the released WLR benchmark loader \citep{banerjee2025efficient}. Petal is adapted from MIOFlow \citep{huguet2022manifold}, and Converging Gaussian from Wasserstein--Fisher--Rao splines \citep{clancy2022wasserstein}. The WLR authors generated Dyngen Tree and Cycle using the dyngen simulator \citep{cannoodt2021spearheading}; the released ten-dimensional representations use PCA for Tree and PHATE for Cycle \citep{moon2019visualizing}.

Embryoid body originates from \citet{moon2019visualizing}. We use the two-dimensional PHATE coordinates in \texttt{eb\_velocity\_v5.npz} distributed with PACE \citep{yu2026pace}; this processed dataset is also distributed with TrajectoryNet \citep{tong2020trajectorynet}. Mouse hematopoiesis originates from \citet{weinreb2020lineage}; we load the two-dimensional release supplied with BranchSBM \citep{tang2026branched}, whose data source is DeepRUOT \citep{zhang2025deepruot}. The thirty-dimensional Veres data \citep{veres2019charting} are likewise loaded from BranchSBM's released CSV.

CITE-seq uses the WLR release of the NeurIPS 2022 Open Problems time-course data \citep{burkhardt2022openproblems,banerjee2025efficient}, represented in 100 dimensions and subsampled by the released loader to 1000 cells per snapshot. \citet{stoeckius2017simultaneous} introduced the CITE-seq assay; that assay reference is distinct from the source of this time-course dataset.

\paragraph{Licenses, data use, and code.}
A related processed CITE-seq archive \citep{tong2023processed} is licensed CC BY 4.0. Our code is released under the MIT license. We reuse public numerical cell-state data and collect no new human data; model inputs contain expression-derived coordinates, not participant names or contact information. The code repository provides NWSP and NWSP-N, pinned dependencies, a YAML configuration, and synthetic toy datasets. It demonstrates both methods but does not reproduce all benchmark tables.

\paragraph{Fixed settings.}\label{sec:nnimpl}\label{sec:methods_budget}
The generator uses a trainable linear map plus a three-hidden-layer, width-128 SiLU network. Its initial output follows the natural spline through retained snapshot means; the linear part uses the first $d_z$ columns of the pooled covariance square root. The loss is given in \cref{eq:loss}. Training uses double precision and 3000 Adam steps at learning rate $3\times10^{-3}$, with cosine decay and gradient clipping at 10. Each step uses projection and independent energy batches of 256, two stratified time samples ($|\mathcal T|=2$), and fitting batches of $\min(128,n_{\min})$ from every retained snapshot, where $n_{\min}=\min_{k\in\mathcal R}n_k$. The first 300 iterations fit snapshots without the acceleration penalty. The acceleration weight is then increased linearly over 20\% of the remaining iterations. We set $\gamma=0.1$, training-loss MMD bandwidths $0.25$--$4$ times the pooled retained-data median distance, and Sinkhorn hyperparameter $\varepsilon=\max(10^{-3},0.02\bar s)$, where $\bar s$ is the mean summed coordinate variance of retained snapshots. The dictionary has all linear and quadratic potentials plus 36 sine/cosine features across three frequency scales, scaled to unit retained-data gradient energy; ridge parameter $\lambda=10^{-4}$. NWSP-N adds a tanh network of depth 2, width 64 and 16 outputs, which are the learned potentials. Gradients and Hessian--vector products use analytic chain rules, including third activation derivatives when backpropagating into the generator.

\paragraph{Alternating updates.}
When acceleration regularization is used, NWSP-N pretrains its potentials after the 300-step generator warmup for 200 steps at learning rate \compactcriticlr, then makes \compactcriticevery{} potential update per generator step at rate \compactcriticlrjoint. Each potential update uses two independent batches at each sampled time and the symmetric estimate below. The generator update holds potential parameters fixed. Generator normalization constants are instead pooled over four sampled times, refreshed every 25 steps, and held fixed during time differentiation. Computational cost is discussed in \cref{sec:cost}.

\paragraph{Potential-learning estimate.}
At a sampled time, draw two independent batches $\{(X_t^{(j,i)},U_t^{(j,i)})\}_{i=1}^{B_\psi}$, $j\in\{1,2\}$. For learned potential $k$, normalize its gradient by
\[
 s_k=\left[\frac1{2B_\psi}\sum_{j=1}^2\sum_{i=1}^{B_\psi}
 \|\nabla g_{\psi,k}(X_t^{(j,i)})\|^2+10^{-12}\right]^{1/2}.
\]
This normalization is differentiated in $\psi$. Write $(\widehat A_j,\widehat b_j)$ for the empirical moments of batch $j$ with these normalized features and $\hat c_j=(\widehat A_j+\lambda I)^{-1}\widehat b_j$. Let
\[
 e_U=\max\left\{\frac1{2B_\psi}\sum_{j=1}^2\sum_{i=1}^{B_\psi}
 \|U_t^{(j,i)}\|^2,\epsilon_U\right\}.
\]
We fix $\epsilon_U$ at the smallest positive normal double-precision value. The symmetric score is
\begin{equation}
 \widehat{\mathcal Q}_t=\frac{\widehat b_2^\top\hat c_1+\widehat b_1^\top\hat c_2}{2e_U},
 \label{eq:potential_batch}
\end{equation}
and $\widehat{\mathcal Q}=|\mathcal T|^{-1}\sum_{t\in\mathcal T}\widehat{\mathcal Q}_t$ is maximized with $\theta$ fixed. Each batch is used once to fit coefficients and once to evaluate coefficients fitted on the other batch. For generator updates, stored scales are refreshed from 256 samples at each of four stratified times and held fixed during physical-time differentiation. The potential network estimates motion; it does not discriminate between generated and observed data.

\paragraph{Estimation and normalization.}\label{sec:estimation}\label{sec:normconfound}
The independent energy batch in \cref{eq:energy_estimate} prevents reuse of the projection samples, but it does not remove bias from the nonlinear coefficient estimates. The stability result in \cref{thm:stability} concerns moment perturbations; energy quadrature introduces a separate sampling error. The divisor $L^2$ is fixed by retained data rather than a model-dependent energy that may vanish at initialization.

\paragraph{Variants and baselines.}
The unregularized spline fits are shared by NWSP and NWSP-N. Flow baselines train for 10000 steps; their six configurations vary width and noise. MFM uses a locally adaptive normal-distribution metric (LAND) for $d\le10$ and a radial-basis-function metric (RBF) for $d\ge30$, variants OT/I, bandwidths $0.5,1,2$, and 101 Euler steps. WLR and four-point use exact transport at subdivision depth seven. Hardware is described in \cref{sec:compute}. With $K=0$, the NWSP-N training procedure reproduces 0 of 118 reference NWSP fits exactly (median per-run difference 12.9\%). These discrepancies limit attribution of small between-method differences.

\subsection{Hyperparameters and evaluation protocol}
\label{sec:selection}\label{sec:deviations}\label{sec:realdetails}
\paragraph{Hyperparameter search.}
NWSP selects the acceleration weight from $\{0,10^{-8},10^{-6},10^{-4},10^{-2},1\}$, covering the unregularized spline and several orders of magnitude of regularization strength. NWSP-N uses the same grid except on CITE-seq, where it considers $\{0,10^{-4}\}$. Each baseline is evaluated over six candidate configurations.

\paragraph{Baseline implementation.}
On embryoid body, Mouse, and Veres, unequal-size large snapshots made exact-transport subdivision infeasible at full support. WLR and four-point therefore use independent $1000$-point retained-snapshot subsamples, with three draws; all withheld targets remain at full support. Equal-size CITE-seq uses exact full-support transport. For MFM, nonfinite RBF precisions are replaced with the median finite precision; alternate-out fits use retained snapshots only and integrate to each gap's midpoint. These baseline completions were made after initial results were available. Veres alternate-out MFM ran on CPU rather than GPU.

\subsection{Prediction metrics}

\paragraph{Scores of one prediction.} Let $\hat\mu=\sum_{i=1}^{n}\mu_i\delta_{x_i}$ be a prediction and $\hat\nu_h=\frac1m\sum_{j=1}^{m}\delta_{y_j}$
the withheld snapshot of holdout $h$. NWSP and the flow baselines predict $n=m$ points with $\mu_i=1/n$; WLR and four-point subdivision
return a weighted measure, scored with its weights. We compare prediction and snapshot in dataset coordinates and in
full dimension, computing Wasserstein distances exactly as the linear programs
\begin{equation}
  W_p(\hat\mu,\hat\nu_h)^p=\min_{\pi\in\Pi(\mu,\nu)}\sum_{i,j}\pi_{ij}\,\norm{x_i-y_j}^p,\qquad p\in\{1,2\},
  \label{eq:wp}
\end{equation}
The norm is Euclidean, $\nu_j=1/m$ and $\Pi(\mu,\nu)$ is the set of couplings of the two weight vectors. We solve the programs with the network simplex method and verify that it did not stop at its iteration cap. The MMD is the plug-in
estimate averaged over a ladder of Gaussian kernels,
\begin{equation}
  \mathrm{MMD}^2(\hat\mu,\hat\nu_h)=\frac15\sum_{\beta}\Big[\sum_{i,i'}\mu_i\mu_{i'}k_\beta(x_i,x_{i'})+\sum_{j,j'}\nu_j\nu_{j'}k_\beta(y_j,y_{j'})
  -2\sum_{i,j}\mu_i\nu_jk_\beta(x_i,y_j)\Big],
  \label{eq:mmd}
\end{equation}
with $k_\beta(x,y)=\exp(-\norm{x-y}^2/2\beta^2)$ and $\beta\in\{0.25,0.5,1,2,4\}\,\varsigma$, where $\varsigma$ is the standard deviation of all
coordinates of the withheld snapshot. The tables report this averaged squared discrepancy. Embryoid body scores use PHATE coordinates standardized with retained frames, so they are not comparable with other datasets, whose scores use all data coordinates.

We average holdouts within each seed, then seeds 0--2; each spread is one sample SD of the three seed averages. Subsampled transport methods use three subsample draws; deterministic methods have no seed SD.

\FloatBarrier
\section{Computational Cost}
\label{sec:cost}
This section collects the computational properties of NWSP and NWSP-N that are referred to in the main text. We report what the construction costs and what we measured; we did not run matched runtime or memory comparisons against the baselines, so the comparison with transport-based methods below concerns scaling, not measured speed.

\paragraph{What does not depend on the snapshots.}
The snapshot loss compares generated and observed samples at the same time with a Sinkhorn divergence and an MMD term, using minibatches of $\min(128,n_{\min})$ points per retained snapshot. No transport coupling between different times is constructed or stored. In contrast, constructing a transport spline or a Wasserstein subdivision requires couplings between full snapshots, whose cost matrices alone have size quadratic in the snapshot sizes. In our benchmark this is why the evaluated exact-transport implementations of WLR and four-point subdivision needed 1000-point subsamples on three datasets (\cref{sec:selection}), whereas NWSP and NWSP-N were trained with minibatches from the full snapshots. This is a difference in the data that can be used, not a measured speedup.

\paragraph{Cost of the acceleration penalty.}
Let $B$ be the projection batch size, $M=M_0+K$ the number of potentials (\cref{eq:nndict}), and $w$ the width of the learned potential network. At each sampled time, forming the moments $\widehat A,\widehat b,\widehat{\dot A},\widehat{\dot b}$ and factoring $\widehat A+\lambda I$ costs $O(BdM^2+M^3)$ in the dense implementation, plus $O(B(w^2d+wdK))$ for a depth-two potential network. The second linear system \eqref{eq:solve_cdot} reuses the factorization of the first \eqref{eq:solve_c}, so it adds only $O(M^2)$, and the generator's gradients are propagated through both solves. Storing the feature evaluations and the Gram matrix takes $O(BdM+M^2)$ memory, apart from generator activations and derivative buffers. These costs depend on the minibatch and the dictionary, not on the number of observed points per snapshot. With $\lambda_a=0$ the projection and potential-update steps are skipped, and training reduces to the snapshot loss with the spline generator.

\paragraph{Dictionary size is the practical limit.}
The fixed dictionary contains all linear and quadratic potentials plus 36 sine/cosine features, so $M_0=d+d(d+1)/2+36$ grows quadratically in the dimension: $41$ at $d=2$, $531$ at $d=30$, and $5186$ at $d=100$ (\cref{tab:nn_datasets}). At $d=100$ the dense Gram matrix and factorization therefore dominate the cost. Reducing this burden, for example by a low-rank or sparse dictionary or by adapting the dictionary size to the batch size, is a natural direction; the stability bound in \cref{thm:stability} indicates how large the moment errors may be for a target accuracy.

\paragraph{Computing infrastructure.}\label{sec:compute}
The main CPU benchmark used AMD EPYC 9754 processors (128 cores per processor) for NWSP and AMD EPYC 9654 processors (96 cores per processor) for NWSP-N. CPU fits used eight numerical threads for NWSP and one for NWSP-N; these are per-process settings, not the total core counts of the machines. The Veres and CITE-seq fits for both methods used NVIDIA H100 GPUs with 80\,GB memory. NWSP-N GPU workers used one to three CPU threads, depending on the job. Some GPU allocations hosted multiple independent workers. The fresh-seed checks retained the corresponding CPU environments; the additional eight-thread NWSP-N check is reported separately.

The raw-acceleration ablation used EPYC 9654 and EPYC 9754 CPU workers and NVIDIA A100 GPUs with 40\,GB or 80\,GB memory, with two completed H100 tasks also included. The A100 hosts used EPYC 7713 and EPYC 7V13 processors, respectively. Raw-acceleration workers used eight numerical threads.

\paragraph{Inference.}
Sampling evaluates $C_\theta$ and the spline basis. It needs no potential fitting, no numerical integration, and no access to the training snapshots.

\FloatBarrier
\section{Proofs and Approximation Conditions}
\label{sec:proofs}\label{sec:standing}
\paragraph{Standing condition (R).}
The lift $X_t$ is almost surely $C^2$ in time, with $U_t=\dot X_t$ and $W_t=\ddot X_t$, and the potentials are $C^3$ in space. On each compact time interval, assume integrable random envelopes for $f_j(X_t)$, its first two pathwise time derivatives, and the entries of $F(X_t)^\top F(X_t)$ and their first time derivatives. In particular, the envelopes cover $\nabla f_j\cdot U_t$, $U_t^\top\nabla^2 f_j U_t+\nabla f_j\cdot W_t$, and the products in \cref{eq:dmoments}. These conditions justify the exchanges of time differentiation and expectation used below. For the least-squares interpretation, assume square-integrable path velocities; whenever acceleration energies are used, assume those energies are finite. None of these finite-dictionary identities requires a density for $\rho_t$.

All dictionaries and feature scales in this appendix are held fixed during time differentiation. Vector norms are Euclidean, matrix norms are induced operator norms, and field norms are $L^2(\rho_t)$ unless otherwise indicated.
\subsection{Representation invariance}
\begin{proof}[Proof of \cref{prop:invariance}]
The chain rule and (R) give
\[
 A_{ij}=\int\nabla f_i\cdot\nabla f_j\,d\rho_t,\qquad
 b_j=\frac{d}{dt}\int f_j\,d\rho_t,\qquad
 \dot b_j=\frac{d^2}{dt^2}\int f_j\,d\rho_t.
\]
Differentiating the first identity gives $\dot A$. Thus two representations of the same marginal curve have the same four moments. Since $A+\lambda I\succeq \lambda I$, both linear systems have unique solutions. The fields $v_M,a_M$ and their population energies consequently coincide. The argument applies unchanged to a frozen learned dictionary.
\end{proof}

\Needspace{17\baselineskip}
\subsection{Stability of the two solves}
\begin{theorem}[Stability under moment errors]
\label{thm:stability}
At a fixed time, let $H=A+\lambda I$ and $\kappa=\lambda_{\min}(H)\ge \lambda$. Use the same dictionary $F$, ridge $\lambda$, and evaluation measure for the perturbed system. Hats denote perturbed moments and the solutions of the two corresponding linear systems; in particular, $\widehat{\dot c}$ denotes the second solution. Write $\delta_A=\|\hat A-A\|<\kappa$, with analogous notation for the other moment errors. Set
\begin{align}
 R_0&=\frac{\delta_b+\delta_A\|c\|}{\kappa-\delta_A},\label{eq:R0}\\
 R_1&=\frac{\delta_{\dot b}+\delta_{\dot A}\|c\|+(\|\dot A\|+\delta_{\dot A})R_0+\delta_A\|\dot c\|}{\kappa-\delta_A}.\label{eq:R1}
\end{align}
Then $\|\hat c-c\|\le R_0$ and $\|\widehat{\dot c}-\dot c\|\le R_1$. Wherever $\|F(x)\|\le K_0$ and $\|D_xF(x)[u]\|\le K_1\|u\|$,
\begin{equation}
 \|\hat a_M-a_M\|\le K_0R_1+K_0K_1(2\|c\|+R_0)R_0.
 \label{eq:accbound}
\end{equation}
\end{theorem}
Since $\kappa\ge\lambda$, the condition $\delta_A<\kappa$ holds whenever the Gram-matrix error is below the ridge parameter. The experiments reported here do not measure $\|\hat A-A\|$ or $\lambda_{\min}(H)$, so the condition is a hypothesis of the theorem, not a verified property of the reported runs.
\begin{proof}
Let $\Delta A=\hat A-A$, define the other moment differences similarly, and put $e=\hat c-c$, $g=\widehat{\dot c}-\dot c$. The inverse of $\hat H=\hat A+\lambda I$ has norm at most $(\kappa-\delta_A)^{-1}$. Subtracting the two pairs of normal equations yields
\[
 \hat He=\Delta b-\Delta A c,\qquad
 \hat Hg=\Delta\dot b-\Delta\dot A c-\widehat{\dot A}e-\Delta A\dot c.
\]
The inverse bound gives $R_0,R_1$. Finally,
$\hat a_M-a_M=Fg+D_xF[F\hat c]e+D_xF[Fe]c$,
which implies \cref{eq:accbound}. If $D=\|\hat a_M-a_M\|_{L^2(\rho_t)}$, the error in squared energy is at most $(2\|a_M\|_{L^2(\rho_t)}+D)D$.
\end{proof}

\subsection{Fourier approximation}
\label{sec:consistencyproof}
\begin{theorem}[Consistency for Fourier trial spaces]
\label{thm:consistency}
On the $d$-dimensional torus $\T^d$ with normalized Lebesgue measure, let $\rho\in C^1([0,1];L^\infty(\T^d))$ satisfy $0<m\le\rho_t\le M_\rho$, and let the tangent velocity $v\in C^1([0,1];H^s(\T^d))$, $s>1+d/2$, be the gradient of a periodic potential. Let $V_N$ be the span of the gradient Fourier modes with $0<\|k\|_\infty\le N$, expressed in an unweighted $L^2$-orthonormal basis. Let $v_N$ be the ridge projection onto $V_N$ and $a_N$ its material derivative. With exact moments and a time-independent ridge $0\le \lambda_N\le C_rN^{-s}$,
\begin{align}
 \sup_t\big(\|v_N-v\|_{L^2}+\|\partial_t v_N-\partial_t v\|_{L^2}\big)&\le CN^{-s},\label{eq:vrate}\\
 \sup_t\|a_N-a\|_{L^2(\rho_t)}&\le CN^{1-s}.\label{eq:arate}
\end{align}
The integrated acceleration energies converge at rate $O(N^{1-s})$. Here $N$ is a frequency cutoff, not the number $M$ of dictionary columns.
\end{theorem}
\begin{proof}
Write $\|\cdot\|_2$ for the unweighted spatial norm, let $P_N$ be the unweighted Fourier projection onto $V_N$, and define $\mathfrak b_t(u,w)=\int\rho_tu\cdot w$. Constants below are independent of $t$ and $N$. Orthonormality makes the coefficient ridge equal to $\lambda_N\|v_N\|_2^2$, so the normal equations are
\[
 \mathfrak b_t(v_N,w)+\lambda_N\langle v_N,w\rangle
 =\mathfrak b_t(v,w),\qquad w\in V_N.
\]
For $u=v,\partial_t v$, Fourier truncation gives
$\|(I-P_N)u\|_2\le CN^{-s}$ and $\|(I-P_N)u\|_{H^1}\le CN^{1-s}$.

\emph{Velocity.} Put $q_N=v_N-P_Nv$. Subtract the equation for $P_Nv$ and test with $q_N$:
\[
 m\|q_N\|_2^2\le
 \bigl(M_\rho\|v-P_Nv\|_2+\lambda_N\|P_Nv\|_2\bigr)\|q_N\|_2.
\]
Hence $\|q_N\|_2\le CN^{-s}$. For $q\in V_N$, the inverse bounds are
$\|q\|_{H^1}\le CN\|q\|_2$ and $\|q\|_\infty\le CN^{d/2}\|q\|_2$.
Together with truncation and Sobolev embedding, they give, for $e_N=v_N-v$,
\[
 \|e_N\|_2\le CN^{-s},\qquad
 \|e_N\|_{H^1}\le CN^{1-s},\qquad
 \|v_N\|_\infty\le C.
\]

\emph{Time derivative.} Because the basis and ridge are time-independent, differentiation gives
\[
 \mathfrak b_t(\partial_t v_N,w)+\lambda_N\langle\partial_t v_N,w\rangle
 =\mathfrak b_t(\partial_t v,w)-\int\partial_t\rho_t\,e_N\cdot w.
\]
Test with $q_N'=\partial_t v_N-P_N\partial_t v$. Coercivity bounds its norm by
\[
 \|q_N'\|_2\le C\bigl(
 \|(I-P_N)\partial_t v\|_2+\lambda_N\|P_N\partial_t v\|_2
 +\|\partial_t\rho_t\|_\infty\|e_N\|_2\bigr)\le CN^{-s}.
\]
This proves \cref{eq:vrate} after adding the truncation error.

\emph{Acceleration and energy.} For $e_M=v_M-v$, the exact decomposition
\begin{equation}
 a_M-a=\partial_t e_M+D_xe_M\,v_M+D_xv\,e_M
 \label{eq:accdecomp}
\end{equation}
implies
\[
 \|a_N-a\|_2\le\|\partial_t e_N\|_2
 +\|D_xe_N\|_2\|v_N\|_\infty
 +\|D_xv\|_\infty\|e_N\|_2\le CN^{1-s}.
\]
The derivative bound follows from $s>1+d/2$. Equivalence of weighted and unweighted norms proves \cref{eq:arate}. Finally,
\[
 \big|\|a_N\|_{\rho_t}^2-\|a\|_{\rho_t}^2\big|
 \le\bigl(2\|a\|_{\rho_t}+\|a_N-a\|_{\rho_t}\bigr)\|a_N-a\|_{\rho_t}.
\]
Uniform bounds and integration over $[0,1]$ give the energy rate.
\end{proof}

\subsection{Why the moving-Gram correction cannot be dropped}
\label{sec:periodiccalc}
\begin{corollary}[Error from omitting the moving-Gram correction]
\label[corollary]{cor:gram_bias}
Let $\dot c^{\rm omit}$ solve $(A+\lambda I)\dot c^{\rm omit}=\dot b$ and set
$a_M^{\rm omit}=F\dot c^{\rm omit}+D_xF[Fc]c$. For any fixed dictionary with an invertible system,
\begin{equation}
 a_M^{\rm omit}-a_M=F(A+\lambda I)^{-1}\dot A c.
 \label{eq:omit_bias}
\end{equation}
Under the assumptions of \cref{thm:consistency}, at each fixed time,
\[
 a_N^{\rm omit}\longrightarrow
 a+\Pi_{\rho_t}\bigl[(\partial_t\rho_t/\rho_t)v\bigr]
 \quad\text{in }L^2(\rho_t),
\]
where $\Pi_{\rho_t}$ projects onto the closure of gradients of smooth periodic potentials. The additional term need not vanish.
\end{corollary}
\begin{proof}
Subtract the corrected and uncorrected coefficient equations to obtain \cref{eq:omit_bias}. In the Fourier setting, write the right-hand side as $z_N$. Its variational equation is
\[
 \int\rho_tz_N\cdot w+\lambda_N\langle z_N,w\rangle
 =\int\partial_t\rho_t\,v_N\cdot w,\qquad w\in V_N.
\]
Set $g=(\partial_t\rho_t/\rho_t)v$ and $z=\Pi_{\rho_t}g$. For any $w_N\in V_N$, subtract the projection identity for $z$ and test with $z_N-w_N$. Coercivity and bounded $\partial_t\rho_t$ give
\[
 \|z_N-z\|_{L^2}\le C\bigl(
 \|z-w_N\|_{L^2}+\|v_N-v\|_{L^2}+\lambda_N\|w_N\|_{L^2}\bigr).
\]
The Fourier spaces are dense in the periodic gradient space. Choosing $w_N\to z$, using $v_N\to v$ and $\lambda_N\to0$, and applying norm equivalence proves the limit.

For an explicit nonzero limit, take $\rho_t(x)=1+\alpha(t)\cos x$ on the circle, with $\alpha(t)=\eta\sin(2\pi t)$ and $0<\eta<1$. The periodic gradient velocity $v=-\alpha'(t)\sin x/\rho_t$ satisfies the continuity equation. At $t=0$, $\rho_0=1$ and
\[
 (\partial_t\rho_t/\rho_t)v\big|_{t=0}
 =-\alpha'(0)^2\sin x\cos x=-2\pi^2\eta^2\sin(2x).
\]
This is already a periodic gradient, so its projection is itself and is nonzero. Even an exact velocity fit cannot remove the omitted derivative term.
\end{proof}

\paragraph{Periodic numerical check.}
The numerical check uses $x\in[0,2\pi)$ with reference measure $dx/(2\pi)$, density $\rho_t(x)=1+0.3\sin(2\pi t)\cos x$, velocity $v=-\partial_t[0.3\sin(2\pi t)]\sin x/\rho_t$, zero ridge, and 16,384 quadrature nodes at $t=0.137$. At cutoff $N=16$, the reported absolute error $\|a_N-a\|_{L^2(\rho_t)}$ is $7.3\times10^{-14}$ with the correction and $0.577$ without it. This absolute norm is distinct from the relative-error metric used in the learned-potential experiments.

\subsection{A finite-dictionary resolution example}
\label{sec:mixtureproof}
For $0<m(t)<1$, let $\rho_t=\tfrac12\mathcal N(-m(t),1-m(t)^2)+\tfrac12\mathcal N(m(t),1-m(t)^2)$. Its first two moments are $0$ and $1$, whereas its fourth moment is $3-2m(t)^4$. Nonconstant $m(t)$ therefore changes the distribution without changing the moments of linear or quadratic potentials. For that dictionary, $b=\dot b=0$, hence $c=\dot c=0$ and $a_M=0$. The example separates two properties: an estimator can be representation-invariant and still miss distributional changes outside its projection space. It motivates learning potentials, but does not prove that the implemented learning rule resolves every such change.

\clearpage
\section{Supplementary Prediction Results}
\label{sec:benchmark}
\Cref{tab:full_W2,tab:full_MMD} report $W_2$ and MMD for the same fitted models used in the main $W_1$ comparison. Hyperparameters are not retuned for these metrics. \Cref{tab:dataset-tests} aggregates relative $W_1$ differences within datasets to avoid counting two protocols as independent datasets. These summaries describe the observed score differences; they are not independent replications across protocols.
\begin{table}[!htbp]
\centering
\footnotesize
\caption{\textbf{Held-out $W_2$} (lower is better). Same fitted models, splits, and conventions as \cref{tab:main}.}
\label{tab:full_W2}
\setlength{\tabcolsep}{3.5pt}
\nnfit{%
\begin{tabular}{@{}lrrrrrrrrr@{}}
\toprule
Dataset & WLR & 4-point & OT-CFM & SB-CFM & CFM & {[SF]}$^{2}$M & MFM & \ourhead{NWSP} & \ourhead{NWSP-N} \\
\midrule
\multicolumn{10}{@{}l}{\emph{Leave-one-out}} \\
\shortstack[l]{Diverging Gauss.\\{\scriptsize\vphantom{$\pm 0$}}} & \shortstack[r]{$\underline{1.333}$\\{\scriptsize \vphantom{$\pm 0$}}} & \shortstack[r]{$1.368$\\{\scriptsize \vphantom{$\pm 0$}}} & \shortstack[r]{$1.341$\\{\scriptsize $\pm\,0.039$}} & \shortstack[r]{$1.351$\\{\scriptsize $\pm\,0.042$}} & \shortstack[r]{$1.371$\\{\scriptsize $\pm\,0.025$}} & \shortstack[r]{$\mathbf{1.321}$\\{\scriptsize $\pm\,0.027$}} & \shortstack[r]{$1.433$\\{\scriptsize $\pm\,0.099$}} & \shortstack[r]{$1.364$\\{\scriptsize $\pm\,0.016$}} & \shortstack[r]{$1.357$\\{\scriptsize $\pm\,0.027$}}  \\
\addlinespace[2pt]
\shortstack[l]{Petal\\{\scriptsize\vphantom{$\pm 0$}}} & \shortstack[r]{$\underline{0.076}$\\{\scriptsize \vphantom{$\pm 0$}}} & \shortstack[r]{$\mathbf{0.074}$\\{\scriptsize \vphantom{$\pm 0$}}} & \shortstack[r]{$0.084$\\{\scriptsize $\pm\,0.003$}} & \shortstack[r]{$0.083$\\{\scriptsize $\pm\,0.002$}} & \shortstack[r]{$0.193$\\{\scriptsize $\pm\,0.034$}} & \shortstack[r]{$0.109$\\{\scriptsize $\pm\,0.004$}} & \shortstack[r]{$0.102$\\{\scriptsize $\pm\,0.005$}} & \shortstack[r]{$0.123$\\{\scriptsize $\pm\,0.022$}} & \shortstack[r]{$0.116$\\{\scriptsize $\pm\,0.016$}}  \\
\addlinespace[2pt]
\shortstack[l]{Converging Gauss.\\{\scriptsize\vphantom{$\pm 0$}}} & \shortstack[r]{$0.488$\\{\scriptsize \vphantom{$\pm 0$}}} & \shortstack[r]{$0.491$\\{\scriptsize \vphantom{$\pm 0$}}} & \shortstack[r]{$0.472$\\{\scriptsize $\pm\,0.011$}} & \shortstack[r]{$\mathbf{0.466}$\\{\scriptsize $\pm\,0.011$}} & \shortstack[r]{$0.508$\\{\scriptsize $\pm\,0.010$}} & \shortstack[r]{$\underline{0.468}$\\{\scriptsize $\pm\,0.010$}} & \shortstack[r]{$0.471$\\{\scriptsize $\pm\,0.040$}} & \shortstack[r]{$0.496$\\{\scriptsize $\pm\,0.023$}} & \shortstack[r]{$0.524$\\{\scriptsize $\pm\,0.034$}}  \\
\addlinespace[2pt]
\shortstack[l]{Embryoid body\\{\scriptsize\vphantom{$\pm 0$}}} & \shortstack[r]{$0.552$\\{\scriptsize \vphantom{$\pm 0$}}} & \shortstack[r]{$0.525$\\{\scriptsize \vphantom{$\pm 0$}}} & \shortstack[r]{$0.582$\\{\scriptsize $\pm\,0.013$}} & \shortstack[r]{$0.569$\\{\scriptsize $\pm\,0.016$}} & \shortstack[r]{$0.601$\\{\scriptsize $\pm\,0.011$}} & \shortstack[r]{$0.569$\\{\scriptsize $\pm\,0.015$}} & \shortstack[r]{$0.495$\\{\scriptsize $\pm\,0.024$}} & \shortstack[r]{$\mathbf{0.448}$\\{\scriptsize $\pm\,0.002$}} & \shortstack[r]{$\underline{0.454}$\\{\scriptsize $\pm\,0.004$}}  \\
\addlinespace[2pt]
\shortstack[l]{CITE-seq\\{\scriptsize\vphantom{$\pm 0$}}} & \shortstack[r]{$32.598$\\{\scriptsize \vphantom{$\pm 0$}}} & \shortstack[r]{$33.245$\\{\scriptsize \vphantom{$\pm 0$}}} & \shortstack[r]{$33.873$\\{\scriptsize $\pm\,0.106$}} & \shortstack[r]{$33.648$\\{\scriptsize $\pm\,0.182$}} & \shortstack[r]{$34.815$\\{\scriptsize $\pm\,0.577$}} & \shortstack[r]{$35.240$\\{\scriptsize $\pm\,0.272$}} & \shortstack[r]{$32.861$\\{\scriptsize $\pm\,0.127$}} & \shortstack[r]{$\underline{32.013}$\\{\scriptsize $\pm\,0.139$}} & \shortstack[r]{$\mathbf{31.992}$\\{\scriptsize $\pm\,0.130$}}  \\
\addlinespace[2pt]
\shortstack[l]{Circular Gauss.\\{\scriptsize\vphantom{$\pm 0$}}} & \shortstack[r]{$1.975$\\{\scriptsize \vphantom{$\pm 0$}}} & \shortstack[r]{$1.425$\\{\scriptsize \vphantom{$\pm 0$}}} & \shortstack[r]{$1.578$\\{\scriptsize $\pm\,0.067$}} & \shortstack[r]{$1.611$\\{\scriptsize $\pm\,0.025$}} & \shortstack[r]{$1.592$\\{\scriptsize $\pm\,0.052$}} & \shortstack[r]{$1.622$\\{\scriptsize $\pm\,0.048$}} & \shortstack[r]{$0.928$\\{\scriptsize $\pm\,0.119$}} & \shortstack[r]{$\underline{0.484}$\\{\scriptsize $\pm\,0.017$}} & \shortstack[r]{$\mathbf{0.482}$\\{\scriptsize $\pm\,0.017$}}  \\
\addlinespace[2pt]
\shortstack[l]{Dyngen Tree\\{\scriptsize\vphantom{$\pm 0$}}} & \shortstack[r]{$2.533$\\{\scriptsize \vphantom{$\pm 0$}}} & \shortstack[r]{$2.529$\\{\scriptsize \vphantom{$\pm 0$}}} & \shortstack[r]{$2.559$\\{\scriptsize $\pm\,0.056$}} & \shortstack[r]{$2.590$\\{\scriptsize $\pm\,0.109$}} & \shortstack[r]{$\mathbf{2.469}$\\{\scriptsize $\pm\,0.088$}} & \shortstack[r]{$2.555$\\{\scriptsize $\pm\,0.046$}} & \shortstack[r]{$2.608$\\{\scriptsize $\pm\,0.054$}} & \shortstack[r]{$\underline{2.480}$\\{\scriptsize $\pm\,0.027$}} & \shortstack[r]{$2.499$\\{\scriptsize $\pm\,0.007$}}  \\
\addlinespace[2pt]
\shortstack[l]{Dyngen Cycle\\{\scriptsize\vphantom{$\pm 0$}}} & \shortstack[r]{$0.933$\\{\scriptsize \vphantom{$\pm 0$}}} & \shortstack[r]{$0.819$\\{\scriptsize \vphantom{$\pm 0$}}} & \shortstack[r]{$0.712$\\{\scriptsize $\pm\,0.007$}} & \shortstack[r]{$0.719$\\{\scriptsize $\pm\,0.013$}} & \shortstack[r]{$0.761$\\{\scriptsize $\pm\,0.008$}} & \shortstack[r]{$0.721$\\{\scriptsize $\pm\,0.023$}} & \shortstack[r]{$0.702$\\{\scriptsize $\pm\,0.010$}} & \shortstack[r]{$\mathbf{0.640}$\\{\scriptsize $\pm\,0.013$}} & \shortstack[r]{$\underline{0.641}$\\{\scriptsize $\pm\,0.013$}}  \\
\addlinespace[2pt]
\shortstack[l]{Mouse hemat.\\{\scriptsize\vphantom{$\pm 0$}}} & \shortstack[r]{$0.498$\\{\scriptsize \vphantom{$\pm 0$}}} & \shortstack[r]{$0.494$\\{\scriptsize \vphantom{$\pm 0$}}} & \shortstack[r]{$0.493$\\{\scriptsize $\pm\,0.005$}} & \shortstack[r]{$0.483$\\{\scriptsize $\pm\,0.014$}} & \shortstack[r]{$0.611$\\{\scriptsize $\pm\,0.013$}} & \shortstack[r]{$\underline{0.479}$\\{\scriptsize $\pm\,0.017$}} & \shortstack[r]{$\mathbf{0.449}$\\{\scriptsize $\pm\,0.004$}} & \shortstack[r]{$0.491$\\{\scriptsize $\pm\,0.006$}} & \shortstack[r]{$0.489$\\{\scriptsize $\pm\,0.010$}}  \\
\addlinespace[2pt]
\shortstack[l]{Veres $\beta$-cell\\{\scriptsize\vphantom{$\pm 0$}}} & \shortstack[r]{$8.784$\\{\scriptsize \vphantom{$\pm 0$}}} & \shortstack[r]{$8.954$\\{\scriptsize \vphantom{$\pm 0$}}} & \shortstack[r]{$9.131$\\{\scriptsize $\pm\,0.065$}} & \shortstack[r]{$9.294$\\{\scriptsize $\pm\,0.094$}} & \shortstack[r]{$9.705$\\{\scriptsize $\pm\,0.071$}} & \shortstack[r]{$9.138$\\{\scriptsize $\pm\,0.053$}} & \shortstack[r]{$8.720$\\{\scriptsize $\pm\,0.016$}} & \shortstack[r]{$\mathbf{8.645}$\\{\scriptsize $\pm\,0.026$}} & \shortstack[r]{$\underline{8.651}$\\{\scriptsize $\pm\,0.028$}}  \\
\addlinespace[2pt]
\midrule
\multicolumn{10}{@{}l}{\emph{Leave-alternate-out}} \\
\shortstack[l]{Petal\\{\scriptsize\vphantom{$\pm 0$}}} & \shortstack[r]{$\mathbf{0.066}$\\{\scriptsize \vphantom{$\pm 0$}}} & \shortstack[r]{$\underline{0.069}$\\{\scriptsize \vphantom{$\pm 0$}}} & \shortstack[r]{$0.134$\\{\scriptsize $\pm\,0.004$}} & \shortstack[r]{$0.084$\\{\scriptsize $\pm\,0.008$}} & \shortstack[r]{$0.206$\\{\scriptsize $\pm\,0.010$}} & \shortstack[r]{$0.104$\\{\scriptsize $\pm\,0.020$}} & \shortstack[r]{$0.086$\\{\scriptsize $\pm\,0.005$}} & \shortstack[r]{$0.124$\\{\scriptsize $\pm\,0.022$}} & \shortstack[r]{$0.122$\\{\scriptsize $\pm\,0.029$}}  \\
\addlinespace[2pt]
\shortstack[l]{Embryoid body\\{\scriptsize\vphantom{$\pm 0$}}} & \shortstack[r]{$0.534$\\{\scriptsize \vphantom{$\pm 0$}}} & \shortstack[r]{$0.565$\\{\scriptsize \vphantom{$\pm 0$}}} & \shortstack[r]{$0.573$\\{\scriptsize $\pm\,0.007$}} & \shortstack[r]{$0.564$\\{\scriptsize $\pm\,0.037$}} & \shortstack[r]{$0.568$\\{\scriptsize $\pm\,0.036$}} & \shortstack[r]{$0.550$\\{\scriptsize $\pm\,0.008$}} & \shortstack[r]{$0.570$\\{\scriptsize $\pm\,0.025$}} & \shortstack[r]{$\mathbf{0.523}$\\{\scriptsize $\pm\,0.010$}} & \shortstack[r]{$\underline{0.526}$\\{\scriptsize $\pm\,0.010$}}  \\
\addlinespace[2pt]
\shortstack[l]{Circular Gauss.\\{\scriptsize\vphantom{$\pm 0$}}} & \shortstack[r]{$1.921$\\{\scriptsize \vphantom{$\pm 0$}}} & \shortstack[r]{$0.811$\\{\scriptsize \vphantom{$\pm 0$}}} & \shortstack[r]{$1.705$\\{\scriptsize $\pm\,0.125$}} & \shortstack[r]{$1.701$\\{\scriptsize $\pm\,0.131$}} & \shortstack[r]{$1.714$\\{\scriptsize $\pm\,0.109$}} & \shortstack[r]{$1.752$\\{\scriptsize $\pm\,0.083$}} & \shortstack[r]{$1.643$\\{\scriptsize $\pm\,0.416$}} & \shortstack[r]{$\underline{0.538}$\\{\scriptsize $\pm\,0.010$}} & \shortstack[r]{$\mathbf{0.534}$\\{\scriptsize $\pm\,0.016$}}  \\
\addlinespace[2pt]
\shortstack[l]{Dyngen Tree\\{\scriptsize\vphantom{$\pm 0$}}} & \shortstack[r]{$2.726$\\{\scriptsize \vphantom{$\pm 0$}}} & \shortstack[r]{$2.770$\\{\scriptsize \vphantom{$\pm 0$}}} & \shortstack[r]{$2.766$\\{\scriptsize $\pm\,0.019$}} & \shortstack[r]{$2.797$\\{\scriptsize $\pm\,0.031$}} & \shortstack[r]{$\mathbf{2.642}$\\{\scriptsize $\pm\,0.049$}} & \shortstack[r]{$2.718$\\{\scriptsize $\pm\,0.028$}} & \shortstack[r]{$\underline{2.658}$\\{\scriptsize $\pm\,0.017$}} & \shortstack[r]{$2.704$\\{\scriptsize $\pm\,0.013$}} & \shortstack[r]{$2.729$\\{\scriptsize $\pm\,0.023$}}  \\
\addlinespace[2pt]
\shortstack[l]{Dyngen Cycle\\{\scriptsize\vphantom{$\pm 0$}}} & \shortstack[r]{$0.907$\\{\scriptsize \vphantom{$\pm 0$}}} & \shortstack[r]{$0.806$\\{\scriptsize \vphantom{$\pm 0$}}} & \shortstack[r]{$0.869$\\{\scriptsize $\pm\,0.036$}} & \shortstack[r]{$0.890$\\{\scriptsize $\pm\,0.026$}} & \shortstack[r]{$0.865$\\{\scriptsize $\pm\,0.016$}} & \shortstack[r]{$0.875$\\{\scriptsize $\pm\,0.022$}} & \shortstack[r]{$0.828$\\{\scriptsize $\pm\,0.021$}} & \shortstack[r]{$\mathbf{0.764}$\\{\scriptsize $\pm\,0.025$}} & \shortstack[r]{$\underline{0.769}$\\{\scriptsize $\pm\,0.026$}}  \\
\addlinespace[2pt]
\shortstack[l]{Veres $\beta$-cell\\{\scriptsize\vphantom{$\pm 0$}}} & \shortstack[r]{$8.985$\\{\scriptsize \vphantom{$\pm 0$}}} & \shortstack[r]{$9.178$\\{\scriptsize \vphantom{$\pm 0$}}} & \shortstack[r]{$9.208$\\{\scriptsize $\pm\,0.098$}} & \shortstack[r]{$9.241$\\{\scriptsize $\pm\,0.136$}} & \shortstack[r]{$9.645$\\{\scriptsize $\pm\,0.101$}} & \shortstack[r]{$9.318$\\{\scriptsize $\pm\,0.388$}} & \shortstack[r]{$\mathbf{8.557}$\\{\scriptsize $\pm\,0.033$}} & \shortstack[r]{$\underline{8.895}$\\{\scriptsize $\pm\,0.027$}} & \shortstack[r]{$8.907$\\{\scriptsize $\pm\,0.024$}}  \\
\addlinespace[2pt]
\midrule
Best / 16 & 1 & 1 & 0 & 1 & 2 & 1 & 2 & $\mathbf{5}$ & $\underline{3}$ \\
Top two / 16 & 3 & 2 & 0 & 1 & 2 & 3 & 3 & $\mathbf{10}$ & $\underline{8}$ \\
\bottomrule
\end{tabular}}
\end{table}

\begin{table}[!htbp]
\centering
\footnotesize
\caption{\textbf{Held-out MMD$^2$} (lower is better; average over five Gaussian bandwidths). Same fitted models, splits, and conventions as \cref{tab:main}.}
\label{tab:full_MMD}
\setlength{\tabcolsep}{3.5pt}
\nnfit{%
\begin{tabular}{@{}lrrrrrrrrr@{}}
\toprule
Dataset & WLR & 4-point & OT-CFM & SB-CFM & CFM & {[SF]}$^{2}$M & MFM & \ourhead{NWSP} & \ourhead{NWSP-N} \\
\midrule
\multicolumn{10}{@{}l}{\emph{Leave-one-out}} \\
\shortstack[l]{Diverging Gauss.\\{\scriptsize\vphantom{$\pm 0$}}} & \shortstack[r]{$0.278$\\{\scriptsize \vphantom{$\pm 0$}}} & \shortstack[r]{$0.263$\\{\scriptsize \vphantom{$\pm 0$}}} & \shortstack[r]{$0.278$\\{\scriptsize $\pm\,0.008$}} & \shortstack[r]{$0.281$\\{\scriptsize $\pm\,0.008$}} & \shortstack[r]{$0.361$\\{\scriptsize $\pm\,0.002$}} & \shortstack[r]{$0.266$\\{\scriptsize $\pm\,0.003$}} & \shortstack[r]{$0.357$\\{\scriptsize $\pm\,0.037$}} & \shortstack[r]{$\underline{0.197}$\\{\scriptsize $\pm\,0.001$}} & \shortstack[r]{$\mathbf{0.195}$\\{\scriptsize $\pm\,0.001$}}  \\
\addlinespace[2pt]
\shortstack[l]{Petal\\{\scriptsize\vphantom{$\pm 0$}}} & \shortstack[r]{$\underline{0.006}$\\{\scriptsize \vphantom{$\pm 0$}}} & \shortstack[r]{$\mathbf{0.005}$\\{\scriptsize \vphantom{$\pm 0$}}} & \shortstack[r]{$0.007$\\{\scriptsize $\pm\,0.001$}} & \shortstack[r]{$0.007$\\{\scriptsize $\pm\,0.000$}} & \shortstack[r]{$0.048$\\{\scriptsize $\pm\,0.023$}} & \shortstack[r]{$0.011$\\{\scriptsize $\pm\,0.001$}} & \shortstack[r]{$0.018$\\{\scriptsize $\pm\,0.005$}} & \shortstack[r]{$0.010$\\{\scriptsize $\pm\,0.002$}} & \shortstack[r]{$0.009$\\{\scriptsize $\pm\,0.001$}}  \\
\addlinespace[2pt]
\shortstack[l]{Converging Gauss.\\{\scriptsize\vphantom{$\pm 0$}}} & \shortstack[r]{$0.178$\\{\scriptsize \vphantom{$\pm 0$}}} & \shortstack[r]{$0.166$\\{\scriptsize \vphantom{$\pm 0$}}} & \shortstack[r]{$0.173$\\{\scriptsize $\pm\,0.007$}} & \shortstack[r]{$0.163$\\{\scriptsize $\pm\,0.010$}} & \shortstack[r]{$0.194$\\{\scriptsize $\pm\,0.002$}} & \shortstack[r]{$\underline{0.159}$\\{\scriptsize $\pm\,0.009$}} & \shortstack[r]{$\mathbf{0.153}$\\{\scriptsize $\pm\,0.031$}} & \shortstack[r]{$0.169$\\{\scriptsize $\pm\,0.007$}} & \shortstack[r]{$0.181$\\{\scriptsize $\pm\,0.007$}}  \\
\addlinespace[2pt]
\shortstack[l]{Embryoid body\\{\scriptsize\vphantom{$\pm 0$}}} & \shortstack[r]{$0.053$\\{\scriptsize \vphantom{$\pm 0$}}} & \shortstack[r]{$0.051$\\{\scriptsize \vphantom{$\pm 0$}}} & \shortstack[r]{$0.068$\\{\scriptsize $\pm\,0.007$}} & \shortstack[r]{$0.068$\\{\scriptsize $\pm\,0.005$}} & \shortstack[r]{$0.091$\\{\scriptsize $\pm\,0.014$}} & \shortstack[r]{$0.062$\\{\scriptsize $\pm\,0.010$}} & \shortstack[r]{$0.062$\\{\scriptsize $\pm\,0.008$}} & \shortstack[r]{$\mathbf{0.044}$\\{\scriptsize $\pm\,0.002$}} & \shortstack[r]{$\underline{0.045}$\\{\scriptsize $\pm\,0.002$}}  \\
\addlinespace[2pt]
\shortstack[l]{CITE-seq\\{\scriptsize\vphantom{$\pm 0$}}} & \shortstack[r]{$0.007$\\{\scriptsize \vphantom{$\pm 0$}}} & \shortstack[r]{$0.006$\\{\scriptsize \vphantom{$\pm 0$}}} & \shortstack[r]{$0.006$\\{\scriptsize $\pm\,0.000$}} & \shortstack[r]{$\underline{0.006}$\\{\scriptsize $\pm\,0.000$}} & \shortstack[r]{$0.009$\\{\scriptsize $\pm\,0.000$}} & \shortstack[r]{$\mathbf{0.005}$\\{\scriptsize $\pm\,0.000$}} & \shortstack[r]{$0.019$\\{\scriptsize $\pm\,0.004$}} & \shortstack[r]{$0.028$\\{\scriptsize $\pm\,0.001$}} & \shortstack[r]{$0.028$\\{\scriptsize $\pm\,0.001$}}  \\
\addlinespace[2pt]
\shortstack[l]{Circular Gauss.\\{\scriptsize\vphantom{$\pm 0$}}} & \shortstack[r]{$0.273$\\{\scriptsize \vphantom{$\pm 0$}}} & \shortstack[r]{$0.161$\\{\scriptsize \vphantom{$\pm 0$}}} & \shortstack[r]{$0.199$\\{\scriptsize $\pm\,0.015$}} & \shortstack[r]{$0.207$\\{\scriptsize $\pm\,0.005$}} & \shortstack[r]{$0.205$\\{\scriptsize $\pm\,0.014$}} & \shortstack[r]{$0.202$\\{\scriptsize $\pm\,0.014$}} & \shortstack[r]{$0.084$\\{\scriptsize $\pm\,0.021$}} & \shortstack[r]{$\mathbf{0.020}$\\{\scriptsize $\pm\,0.002$}} & \shortstack[r]{$\underline{0.021}$\\{\scriptsize $\pm\,0.001$}}  \\
\addlinespace[2pt]
\shortstack[l]{Dyngen Tree\\{\scriptsize\vphantom{$\pm 0$}}} & \shortstack[r]{$0.109$\\{\scriptsize \vphantom{$\pm 0$}}} & \shortstack[r]{$\mathbf{0.104}$\\{\scriptsize \vphantom{$\pm 0$}}} & \shortstack[r]{$0.116$\\{\scriptsize $\pm\,0.007$}} & \shortstack[r]{$0.116$\\{\scriptsize $\pm\,0.010$}} & \shortstack[r]{$0.120$\\{\scriptsize $\pm\,0.011$}} & \shortstack[r]{$\underline{0.105}$\\{\scriptsize $\pm\,0.004$}} & \shortstack[r]{$0.109$\\{\scriptsize $\pm\,0.006$}} & \shortstack[r]{$0.107$\\{\scriptsize $\pm\,0.003$}} & \shortstack[r]{$0.113$\\{\scriptsize $\pm\,0.003$}}  \\
\addlinespace[2pt]
\shortstack[l]{Dyngen Cycle\\{\scriptsize\vphantom{$\pm 0$}}} & \shortstack[r]{$0.427$\\{\scriptsize \vphantom{$\pm 0$}}} & \shortstack[r]{$0.347$\\{\scriptsize \vphantom{$\pm 0$}}} & \shortstack[r]{$0.311$\\{\scriptsize $\pm\,0.002$}} & \shortstack[r]{$0.314$\\{\scriptsize $\pm\,0.004$}} & \shortstack[r]{$0.326$\\{\scriptsize $\pm\,0.013$}} & \shortstack[r]{$\mathbf{0.254}$\\{\scriptsize $\pm\,0.007$}} & \shortstack[r]{$0.300$\\{\scriptsize $\pm\,0.003$}} & \shortstack[r]{$\underline{0.267}$\\{\scriptsize $\pm\,0.012$}} & \shortstack[r]{$0.268$\\{\scriptsize $\pm\,0.013$}}  \\
\addlinespace[2pt]
\shortstack[l]{Mouse hemat.\\{\scriptsize\vphantom{$\pm 0$}}} & \shortstack[r]{$0.059$\\{\scriptsize \vphantom{$\pm 0$}}} & \shortstack[r]{$0.058$\\{\scriptsize \vphantom{$\pm 0$}}} & \shortstack[r]{$0.062$\\{\scriptsize $\pm\,0.004$}} & \shortstack[r]{$0.062$\\{\scriptsize $\pm\,0.007$}} & \shortstack[r]{$0.112$\\{\scriptsize $\pm\,0.004$}} & \shortstack[r]{$\mathbf{0.056}$\\{\scriptsize $\pm\,0.005$}} & \shortstack[r]{$\underline{0.056}$\\{\scriptsize $\pm\,0.006$}} & \shortstack[r]{$0.060$\\{\scriptsize $\pm\,0.001$}} & \shortstack[r]{$0.059$\\{\scriptsize $\pm\,0.002$}}  \\
\addlinespace[2pt]
\shortstack[l]{Veres $\beta$-cell\\{\scriptsize\vphantom{$\pm 0$}}} & \shortstack[r]{$0.008$\\{\scriptsize \vphantom{$\pm 0$}}} & \shortstack[r]{$\mathbf{0.007}$\\{\scriptsize \vphantom{$\pm 0$}}} & \shortstack[r]{$0.009$\\{\scriptsize $\pm\,0.001$}} & \shortstack[r]{$0.008$\\{\scriptsize $\pm\,0.000$}} & \shortstack[r]{$0.012$\\{\scriptsize $\pm\,0.000$}} & \shortstack[r]{$\underline{0.008}$\\{\scriptsize $\pm\,0.000$}} & \shortstack[r]{$0.015$\\{\scriptsize $\pm\,0.001$}} & \shortstack[r]{$0.022$\\{\scriptsize $\pm\,0.000$}} & \shortstack[r]{$0.023$\\{\scriptsize $\pm\,0.000$}}  \\
\addlinespace[2pt]
\midrule
\multicolumn{10}{@{}l}{\emph{Leave-alternate-out}} \\
\shortstack[l]{Petal\\{\scriptsize\vphantom{$\pm 0$}}} & \shortstack[r]{$\mathbf{0.002}$\\{\scriptsize \vphantom{$\pm 0$}}} & \shortstack[r]{$\underline{0.003}$\\{\scriptsize \vphantom{$\pm 0$}}} & \shortstack[r]{$0.063$\\{\scriptsize $\pm\,0.006$}} & \shortstack[r]{$0.007$\\{\scriptsize $\pm\,0.002$}} & \shortstack[r]{$0.111$\\{\scriptsize $\pm\,0.010$}} & \shortstack[r]{$0.009$\\{\scriptsize $\pm\,0.005$}} & \shortstack[r]{$0.009$\\{\scriptsize $\pm\,0.001$}} & \shortstack[r]{$0.009$\\{\scriptsize $\pm\,0.002$}} & \shortstack[r]{$0.008$\\{\scriptsize $\pm\,0.003$}}  \\
\addlinespace[2pt]
\shortstack[l]{Embryoid body\\{\scriptsize\vphantom{$\pm 0$}}} & \shortstack[r]{$\mathbf{0.045}$\\{\scriptsize \vphantom{$\pm 0$}}} & \shortstack[r]{$0.063$\\{\scriptsize \vphantom{$\pm 0$}}} & \shortstack[r]{$0.060$\\{\scriptsize $\pm\,0.006$}} & \shortstack[r]{$0.071$\\{\scriptsize $\pm\,0.009$}} & \shortstack[r]{$0.095$\\{\scriptsize $\pm\,0.026$}} & \shortstack[r]{$\underline{0.046}$\\{\scriptsize $\pm\,0.004$}} & \shortstack[r]{$0.095$\\{\scriptsize $\pm\,0.012$}} & \shortstack[r]{$0.058$\\{\scriptsize $\pm\,0.004$}} & \shortstack[r]{$0.057$\\{\scriptsize $\pm\,0.002$}}  \\
\addlinespace[2pt]
\shortstack[l]{Circular Gauss.\\{\scriptsize\vphantom{$\pm 0$}}} & \shortstack[r]{$0.270$\\{\scriptsize \vphantom{$\pm 0$}}} & \shortstack[r]{$0.071$\\{\scriptsize \vphantom{$\pm 0$}}} & \shortstack[r]{$0.219$\\{\scriptsize $\pm\,0.022$}} & \shortstack[r]{$0.220$\\{\scriptsize $\pm\,0.026$}} & \shortstack[r]{$0.228$\\{\scriptsize $\pm\,0.019$}} & \shortstack[r]{$0.223$\\{\scriptsize $\pm\,0.015$}} & \shortstack[r]{$0.204$\\{\scriptsize $\pm\,0.080$}} & \shortstack[r]{$\underline{0.033}$\\{\scriptsize $\pm\,0.005$}} & \shortstack[r]{$\mathbf{0.032}$\\{\scriptsize $\pm\,0.004$}}  \\
\addlinespace[2pt]
\shortstack[l]{Dyngen Tree\\{\scriptsize\vphantom{$\pm 0$}}} & \shortstack[r]{$\underline{0.115}$\\{\scriptsize \vphantom{$\pm 0$}}} & \shortstack[r]{$\underline{0.115}$\\{\scriptsize \vphantom{$\pm 0$}}} & \shortstack[r]{$0.125$\\{\scriptsize $\pm\,0.004$}} & \shortstack[r]{$0.116$\\{\scriptsize $\pm\,0.002$}} & \shortstack[r]{$0.143$\\{\scriptsize $\pm\,0.010$}} & \shortstack[r]{$\mathbf{0.111}$\\{\scriptsize $\pm\,0.002$}} & \shortstack[r]{$0.136$\\{\scriptsize $\pm\,0.002$}} & \shortstack[r]{$0.127$\\{\scriptsize $\pm\,0.003$}} & \shortstack[r]{$0.133$\\{\scriptsize $\pm\,0.005$}}  \\
\addlinespace[2pt]
\shortstack[l]{Dyngen Cycle\\{\scriptsize\vphantom{$\pm 0$}}} & \shortstack[r]{$0.428$\\{\scriptsize \vphantom{$\pm 0$}}} & \shortstack[r]{$\mathbf{0.335}$\\{\scriptsize \vphantom{$\pm 0$}}} & \shortstack[r]{$0.412$\\{\scriptsize $\pm\,0.026$}} & \shortstack[r]{$0.430$\\{\scriptsize $\pm\,0.020$}} & \shortstack[r]{$0.427$\\{\scriptsize $\pm\,0.016$}} & \shortstack[r]{$\underline{0.343}$\\{\scriptsize $\pm\,0.020$}} & \shortstack[r]{$0.386$\\{\scriptsize $\pm\,0.019$}} & \shortstack[r]{$0.345$\\{\scriptsize $\pm\,0.008$}} & \shortstack[r]{$0.345$\\{\scriptsize $\pm\,0.007$}}  \\
\addlinespace[2pt]
\shortstack[l]{Veres $\beta$-cell\\{\scriptsize\vphantom{$\pm 0$}}} & \shortstack[r]{$\underline{0.010}$\\{\scriptsize \vphantom{$\pm 0$}}} & \shortstack[r]{$\mathbf{0.009}$\\{\scriptsize \vphantom{$\pm 0$}}} & \shortstack[r]{$0.011$\\{\scriptsize $\pm\,0.000$}} & \shortstack[r]{$0.011$\\{\scriptsize $\pm\,0.003$}} & \shortstack[r]{$0.017$\\{\scriptsize $\pm\,0.001$}} & \shortstack[r]{$0.011$\\{\scriptsize $\pm\,0.004$}} & \shortstack[r]{$0.015$\\{\scriptsize $\pm\,0.002$}} & \shortstack[r]{$0.020$\\{\scriptsize $\pm\,0.000$}} & \shortstack[r]{$0.020$\\{\scriptsize $\pm\,0.000$}}  \\
\addlinespace[2pt]
\midrule
Best / 16 & 2 & $\mathbf{5}$ & 0 & 0 & 0 & $\underline{4}$ & 1 & 2 & 2 \\
Top two / 16 & 5 & $\underline{7}$ & 0 & 1 & 0 & $\mathbf{9}$ & 2 & 5 & 4 \\
\bottomrule
\end{tabular}}
\end{table}

\begin{table}[!htbp]
\centering\footnotesize
\setlength{\tabcolsep}{4pt}
\caption{\textbf{Dataset-level comparison with each baseline.} Relative $W_1$ differences are averaged over protocols within each dataset before counting wins. Columns: win count, median relative change (negative favors ours), and one-sided exact sign-test $p$ (post hoc, uncorrected). Bold and underline compare win counts and medians of our two variants.}
\label{tab:dataset-tests}
\begin{tabular}{@{}lrrrrrr@{}}
\toprule
& \multicolumn{3}{c}{\textbf{NWSP (ours)}} & \multicolumn{3}{c}{\textbf{NWSP-N (ours)}} \\
\cmidrule(lr){2-4}\cmidrule(lr){5-7}
Baseline & wins/$n$ & median & $p$ & wins/$n$ & median & $p$ \\
\midrule
WLR & $\mathbf{9/10}$ & $\mathbf{-2.3}$\% & 0.011 & $\underline{8/10}$ & $\mathbf{-2.3}$\% & 0.055  \\
4-point & $\mathbf{8/10}$ & $\mathbf{-4.2}$\% & 0.055 & $\mathbf{8/10}$ & $\mathbf{-4.2}$\% & 0.055  \\
OT-CFM & $\mathbf{8/10}$ & $\underline{-4.2}$\% & 0.055 & $\mathbf{8/10}$ & $\mathbf{-4.9}$\% & 0.055  \\
SB-CFM & $\mathbf{8/10}$ & $\underline{-5.0}$\% & 0.055 & $\mathbf{8/10}$ & $\mathbf{-5.7}$\% & 0.055  \\
CFM & $\mathbf{9/10}$ & $\underline{-12.8}$\% & 0.011 & $\mathbf{9/10}$ & $\mathbf{-13.5}$\% & 0.011  \\
{[SF]}$^{2}$M & $\mathbf{9/10}$ & $\underline{-4.5}$\% & 0.011 & $\mathbf{9/10}$ & $\mathbf{-7.1}$\% & 0.011  \\
MFM & $\mathbf{7/10}$ & $\mathbf{-2.3}$\% & 0.172 & $\underline{6/10}$ & $\underline{-1.8}$\% & 0.377  \\
\bottomrule
\end{tabular}
\end{table}

\FloatBarrier

\clearpage
\section{Acceleration Estimation and Regularization}
\label{sec:diagnostics}\label{sec:analyticdetails}

\paragraph{Rotating particle representations.} The latent-rotation check evaluates alternative particle representations at $t=0.5$ while preserving the complete marginal curve. The quoted projected-energy variation uses the largest batch size in the supplied evaluation. Rapid rotations make the cancellations in the differentiated moments difficult to estimate. This tests finite-sample estimation of the same population quantity, not held-out prediction or stationarity of the distribution.

\paragraph{Analytic reference curves.} The analytic examples contain rotating three-well gradient flows, blind Gaussian mixtures and rotating lifts of constant clouds. The reference accelerations are evaluated analytically and checked numerically. Potential training uses 300 steps at rate $3\times10^{-3}$, batches of 256 from 20000 samples. Independent projection and evaluation batches have 4000 and 8000 samples. Reported relative errors average over 21 time points. Tables report five-seed means and SDs. These examples also informed method development.
\begin{table}[!htbp]
\centering
\footnotesize
\caption{\textbf{Optimization and dictionary size for NWSP-N.} Relative acceleration error as in \cref{tab:ana_e1} (mean $\pm$ SD over five seeds); the first column repeats the standard configuration, the others use more updates or more learned potentials. Bold and underline mark the best and second-best means.}
\label{tab:ana_budget}
\setlength{\tabcolsep}{10pt}
\renewcommand{\arraystretch}{1.08}
\begin{tabular}{@{}lrrr@{}}
\toprule
\multicolumn{4}{@{}l}{\textbf{NWSP-N (ours): optimization and dictionary size}} \\
$d$ & \shortstack{16 learned\\300 steps} & \shortstack{16 learned\\3000 steps} & \shortstack{64 learned\\300 steps} \\
\midrule
2 & $0.062\pm0.004$ & $\underline{0.038}\pm0.005$ & $\mathbf{0.032}\pm0.003$ \\
5 & $0.137\pm0.016$ & $\mathbf{0.018}\pm0.002$ & $\underline{0.112}\pm0.008$ \\
10 & $0.199\pm0.016$ & $\mathbf{0.019}\pm0.002$ & $\underline{0.195}\pm0.012$ \\
30 & $\underline{0.341}\pm0.021$ & $\mathbf{0.027}\pm0.003$ & $0.347\pm0.021$ \\
\bottomrule
\end{tabular}
\end{table}

The auxiliary controls also identify limits of adaptation. An untrained block beats the trained block on the one-dimensional blind mixture; long constant-rate runs can degrade or diverge. An offline Action-Matching network overtakes NWSP-N between 300 and 3000 steps, reaching errors \ensuremath{3.4{\times}10^{-4}}--\ensuremath{1.0{\times}10^{-3}} at $10^4$ steps. This comparison does not test either estimator inside a moving generator.

\newpage
\subsection{Raw versus projected acceleration}
\label{sec:explore_raw}
The raw-acceleration ablation uses the same generator, training schedule, regularization search, data scale, and seeds as NWSP in all sixteen settings. It was conducted after the main benchmark. Raw acceleration gives the lower/higher $W_1$ in 5/11 settings (median change $+0.8$\%, mean $+2.4$\%). The counts are 4/12 for $W_2$ and 6/10 for MMD. Differences are below $1\%$ in 6 settings. The raw arm ran on different hardware environments, which can move individual scores by about the size of a seed spread.
\begin{table}[H]
\centering
\footnotesize
\caption{\textbf{Population versus particle acceleration.} Held-out $W_1$ (mean $\pm$ SD over three seeds) with the projected penalty (NWSP) or a raw particle-acceleration penalty, otherwise identical. Last column: change from projected to raw (positive favors NWSP). The two arms ran on different hardware.}
\label{tab:rawmatch}
\setlength{\tabcolsep}{4pt}
\begin{tabular}{@{}lrrr@{}}
\toprule
Dataset & \shortstack{NWSP (ours)\\Projected acceleration} & Raw acceleration & Raw vs projected \\
\midrule
\multicolumn{4}{@{}l}{\emph{Leave-one-out}} \\
Diverging Gauss. & $\underline{1.136}${\scriptsize$\,\pm\,0.017$} & $\mathbf{1.110}${\scriptsize$\,\pm\,0.014$} & $-2.3$\% \\
Petal & $\mathbf{0.090}${\scriptsize$\,\pm\,0.014$} & $\underline{0.100}${\scriptsize$\,\pm\,0.011$} & $+11.8$\% \\
Converging Gauss. & $\mathbf{0.452}${\scriptsize$\,\pm\,0.022$} & $\underline{0.498}${\scriptsize$\,\pm\,0.022$} & $+10.1$\% \\
Embryoid body & $\mathbf{0.350}${\scriptsize$\,\pm\,0.007$} & $\underline{0.361}${\scriptsize$\,\pm\,0.006$} & $+3.2$\% \\
CITE-seq & $\underline{30.545}${\scriptsize$\,\pm\,0.106$} & $\mathbf{30.539}${\scriptsize$\,\pm\,0.139$} & $-0.0$\% \\
Circular Gauss. & $\mathbf{0.427}${\scriptsize$\,\pm\,0.015$} & $\underline{0.428}${\scriptsize$\,\pm\,0.018$} & $+0.3$\% \\
Dyngen Tree & $\mathbf{2.388}${\scriptsize$\,\pm\,0.031$} & $\underline{2.418}${\scriptsize$\,\pm\,0.019$} & $+1.3$\% \\
Dyngen Cycle & $\mathbf{0.607}${\scriptsize$\,\pm\,0.012$} & $\underline{0.618}${\scriptsize$\,\pm\,0.004$} & $+1.8$\% \\
Mouse hemat. & $\mathbf{0.399}${\scriptsize$\,\pm\,0.007$} & $\underline{0.425}${\scriptsize$\,\pm\,0.011$} & $+6.5$\% \\
Veres $\beta$-cell & $\mathbf{7.935}${\scriptsize$\,\pm\,0.008$} & $\underline{7.972}${\scriptsize$\,\pm\,0.004$} & $+0.5$\% \\
\midrule
\multicolumn{4}{@{}l}{\emph{Leave-alternate-out}} \\
Petal & $\mathbf{0.086}${\scriptsize$\,\pm\,0.011$} & $\underline{0.091}${\scriptsize$\,\pm\,0.014$} & $+5.9$\% \\
Embryoid body & $\mathbf{0.412}${\scriptsize$\,\pm\,0.007$} & $\underline{0.413}${\scriptsize$\,\pm\,0.010$} & $+0.3$\% \\
Circular Gauss. & $\underline{0.484}${\scriptsize$\,\pm\,0.007$} & $\mathbf{0.483}${\scriptsize$\,\pm\,0.010$} & $-0.1$\% \\
Dyngen Tree & $\mathbf{2.624}${\scriptsize$\,\pm\,0.011$} & $\underline{2.652}${\scriptsize$\,\pm\,0.022$} & $+1.1$\% \\
Dyngen Cycle & $\underline{0.738}${\scriptsize$\,\pm\,0.019$} & $\mathbf{0.724}${\scriptsize$\,\pm\,0.004$} & $-1.9$\% \\
Veres $\beta$-cell & $\underline{8.153}${\scriptsize$\,\pm\,0.029$} & $\mathbf{8.139}${\scriptsize$\,\pm\,0.006$} & $-0.2$\% \\
\bottomrule
\end{tabular}
\end{table}

\FloatBarrier
\clearpage
\section{Qualitative Results}
\label{sec:qualitative_results}
\Cref{fig:traj_synthetic} illustrates fitted synthetic dynamics using all snapshots; it is not a held-out evaluation. \Cref{fig:sc_heldout_abc,fig:sc_heldout_d} show held-out single-cell predictions. Quantitative comparisons use the full evaluation coordinates, not the two-dimensional plotting projections.

\begin{figure}[!htbp]
\centering\footnotesize
\includegraphics[width=0.98\linewidth]{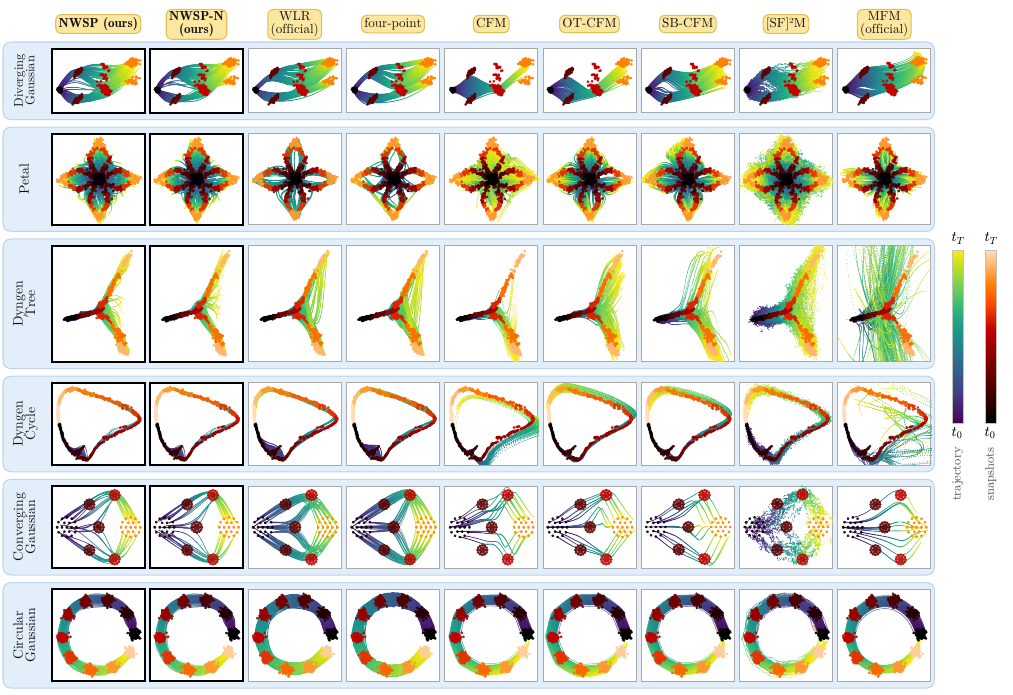}
\caption[All-snapshot trajectories, synthetic datasets]{\textbf{Fitted synthetic population trajectories.} Each model is fitted to all available snapshots; no time point is withheld. Colors indicate time and snapshot index. These plots illustrate the resulting dynamics rather than held-out generalization, which is measured separately in \cref{tab:main}.}
\label{fig:traj_synthetic}\label{sec:qualitative}
\end{figure}

\begin{figure}[!t]
\centering
\includegraphics[height=0.68\textheight]{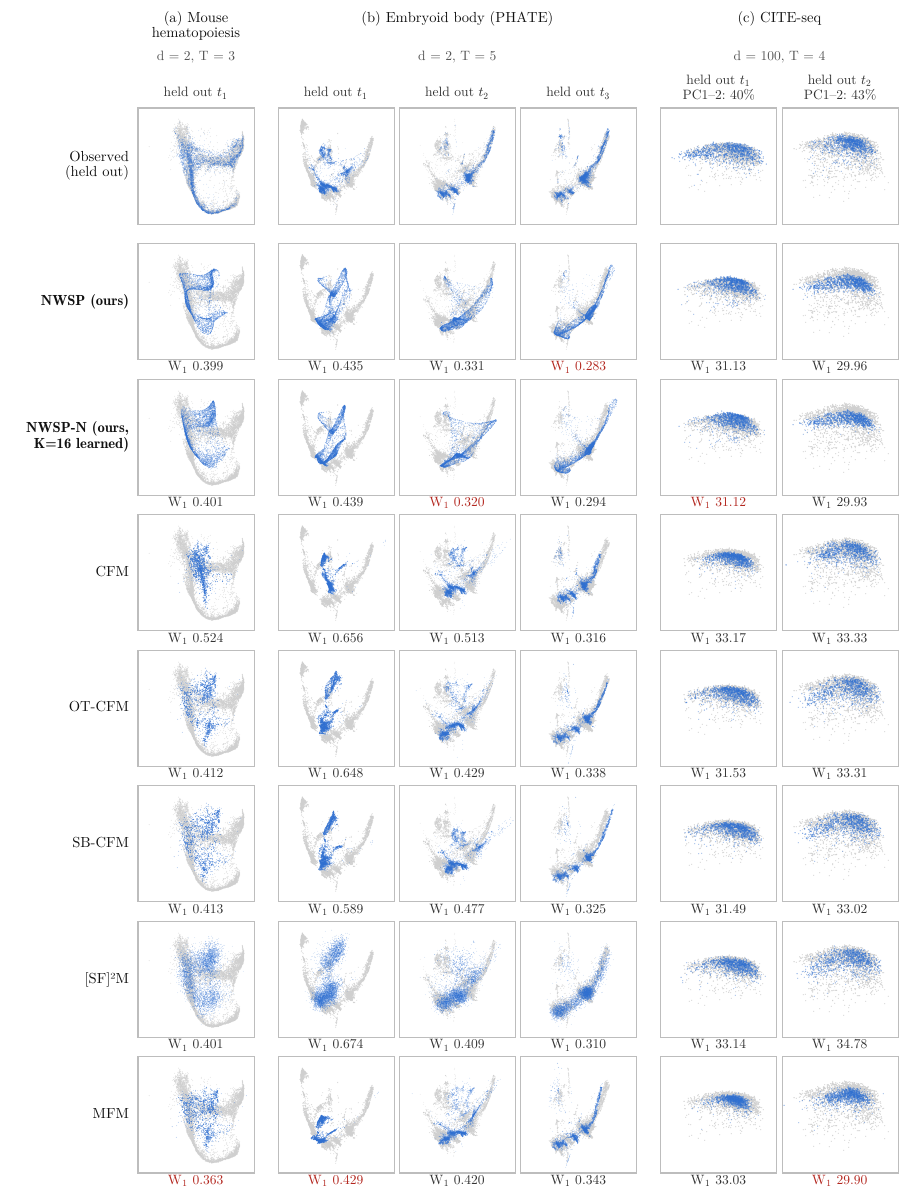}
\caption[Held-out single-cell predictions (Mouse, EB, CITE-seq)]{Held-out snapshots of (a)~Mouse hematopoiesis, (b)~Embryoid body (PHATE)
and (c)~CITE-seq, and each method's prediction of them; every column is a separate leave-one-out fit that never saw its snapshot, and the
number under each panel is that holdout's $W_1$ (red: the lowest of the column). Clouds use seed zero; scores average the three seeds. Gray clouds are retained neighbors. CITE-seq/Veres PCA and Embryoid-body-standardized PHATE use retained frames only. These views do not establish full-dimensional rankings.}
\label{fig:sc_heldout_abc}\label{sec:scdetails}
\end{figure}

\begin{figure}[p]
\centering
\includegraphics[height=0.80\textheight]{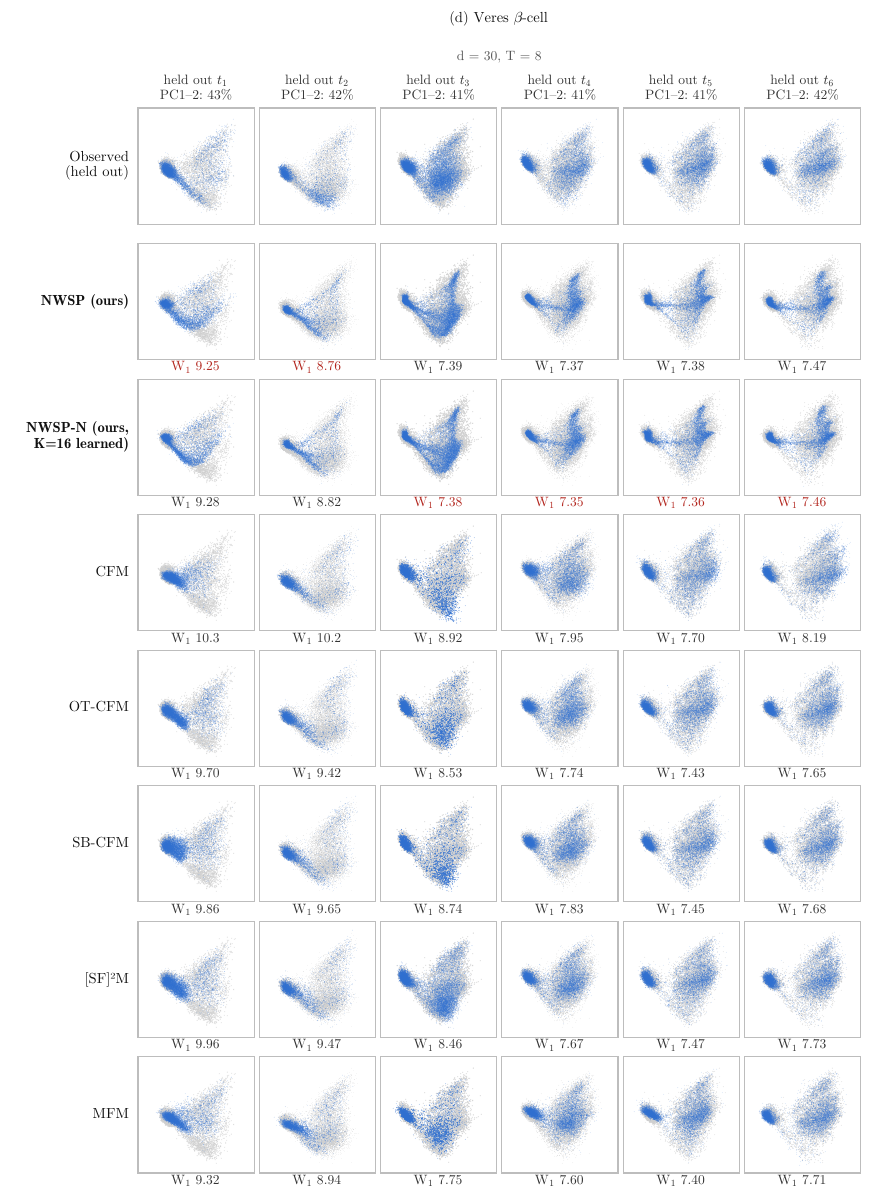}
\caption[Held-out single-cell predictions (Veres)]{Held-out snapshots of (d)~Veres $\beta$-cell and each method's prediction of them, as in
\cref{fig:sc_heldout_abc}.}
\label{fig:sc_heldout_d}
\end{figure}

\end{document}